\documentclass{article}

\PassOptionsToPackage{round}{natbib}

 \usepackage[main, final]{neurips_2026}

\usepackage[utf8]{inputenc} 
\usepackage[T1]{fontenc}    
\usepackage[backref]{hyperref}       
\hypersetup{
    colorlinks,
    linkcolor={purple},
    citecolor={blue!50!black}
}
\usepackage{url}            
\usepackage{booktabs}       
\usepackage{amsfonts}       
\usepackage{nicefrac}       
\usepackage{microtype}      
\usepackage[dvipsnames]{xcolor} 
\usepackage{amsthm}         
\usepackage{amsmath}
\usepackage{mathtools}
\usepackage{bm}  
\usepackage{bbm} 
\usepackage{nicefrac}       
\usepackage{xspace}         
\usepackage{colortbl}       
\usepackage{multirow}       
\usepackage{multicol}
\usepackage{arydshln}       
\usepackage{enumitem}
\usepackage{tabularray}
\usepackage{titletoc}       
\usepackage{newunicodechar}
\UseTblrLibrary{booktabs}
\usepackage[labelfont=bf]{caption}
\usepackage[capitalise]{cleveref}
\newtheorem{theorem}{Theorem}
\newtheorem{definition}{Definition}
\newtheorem{proposition}{Proposition}

\newtheorem{corollary}{Corollary}
\newcommand{\metagame}{\textsc{metagame}}
\newcommand{\metagamex}{\textsc{metagame}\xspace}

\newcommand{\Metagamex}{\textsc{Metagame}\xspace}
\definecolor{customMediumBlue}{rgb}{0.6, 0.6, 1}

\usepackage[most]{tcolorbox}
\definecolor{myred}{HTML}{ff1d62}
\definecolor{myblue}{HTML}{1b7bd0}
\newtcbox{\abox}[3][100]{
  on line,
  colback=#2!#1,
  coltext=#3,
  colframe=#2!#1,
  boxsep=0pt,
  left=2pt, right=2pt,
  top=1.5pt, bottom=1.5pt,
  arc=0pt,
  boxrule=0pt,
}

\title{%
  \resizebox{\linewidth}{!}{The Metagame of Interpretability and Meta-Attributions}
}

\author{%
  Hubert Baniecki\thanks{Corresponding author: \texttt{h.baniecki@uw.edu.pl} \hfill Code: \url{https://github.com/credibleai/metagame}} \hfill Przemyslaw Biecek\\
  University of Warsaw\\
  Centre for Credible AI, Warsaw University of Technology
  \And
  Fabian Fumagalli\\
  Bielefeld University\\
  LMU Munich, MCML
}

\begin{document}

\maketitle

\begin{abstract}
How can an arbitrary attribution method be generalized from first principles to capture interactions? 
We answer this with the \metagame, a conceptual framework for quantifying second-order interaction effects of model explanations. 
We cast the attribution value $\phi_i$ of feature $i$ as a cooperative game among the other features and compute its Shapley value, which measures how much feature $j$ influences the attribution of $i$, yielding the directional \textit{meta-attribution} $\varphi_{j \to i}$. 
By decomposing attribution itself rather than the model directly, meta-attributions extend any gradient- or attention-based method to interactions, uniting removal-based perturbations with model internals.
Theoretically, we prove that meta-attributions sum to the first-order attribution they explain, a hierarchical decomposition that Shapley interactions and integrated Hessians turn out to perform implicitly. 
Empirically, we demonstrate that meta-attributions deliver insights across diverse interpretability applications: (i)~quantifying token interactions in instruction-tuned language models, (ii) explaining cross-modal similarity in vision-language encoders, and (iii)~interpreting text-to-image concepts in multimodal diffusion transformers.
\end{abstract}

\section{Introduction}

Attribution is the task of quantifying how specific variables, such as input tokens, training data, or internal components, influence a machine learning model's behavior. 
Attribution methods serve as the cornerstone of key interpretability approaches, driving advancements in concept-based explainability for vision models~\citep{fel2023holistic}, circuit discovery~\citep{syed2024attribution}, data valuation~\citep{wang2025data}, and instruction-following in language models~\citep{komorowski2026attribution}.
Despite their broad utility, simple attributions based on the notions of cooperative game theory~\citep{lundberg2017unified}, gradients~\citep{sundararajan2017axiomatic}, or relevance propagation~\citep{achtibat2024attnlrp} are limited by capturing only \emph{first-order} influence of a single variable, omitting higher-order dynamics encoded by the model, such as {\color{myred}synergies} and {\color{myblue}antisynergies} (see an example in Figure~\ref{fig:figure1}).
To this end, interaction methods that measure the \emph{second-order} influence of input pairs, such as Shapley interactions~\citep{fumagalli2023shapiq} and integrated Hessians~\citep{janizek2021explaining}, have been proposed.
However, it remains an open challenge \emph{whether}, and \emph{how}, an arbitrary first-order attribution method can be generalized from first principles to capture higher-order interactions.

We answer this question by revisiting the concept of \emph{explaining explanations}~\citep{janizek2021explaining}, generalizing the idea behind their serial use of integrated gradients to quantify how a particular input feature influences an attribution value of another feature. 
\citet{lundstrom2023unifying} have theoretically studied the leakage of interaction effects into individual terms in such \emph{serial} methods.
Furthermore, explicitly computing second-order derivatives such as input Hessians often requires architectural modifications~\citep{dombrowski2019explanations,janizek2021explaining} and specialized computational techniques for large input sizes~\citep{pramanik2026hessianenhanced}. 
Existing \emph{set-based} methods assign interaction scores to groups of features without distinguishing how one feature influences the attribution of another.
We thus propose a general \emph{meta}-approach that addresses these limitations by decomposing an arbitrary first-order attribution into directional interaction terms.

\textbf{This work: The \metagamex, Meta-Shapley value, and meta-attributions.} 
We formally conceptualize the proposed framework as the \metagamex, bridging the discrete principles of cooperative game theory with the continuous dynamics of modern interpretability methods. 
Rather than explaining the model directly, we compute a modified Shapley value targeting the outputs of state-of-the-art attribution methods, such as Grad-ECLIP~\citep{chenyang2024gradeclip}, AttnLRP~\citep{achtibat2024attnlrp}, or ConceptAttention~\citep{helbling2025conceptattention}. 
Crucially, by measuring the marginal contribution of one feature/token strictly in the \emph{presence} of another, we extract \emph{directional} meta-attributions, revealing exactly how one feature/token influences the attribution of another. 
Such a directional focus resolves the interaction leakage observed in serial methods, yielding a rigorous \emph{hierarchical decomposition} of any first-order attribution method.
Theoretically, we prove that these meta-attributions elegantly split classical set-based interactions into their precise, directional components.

\textbf{Contributions.} Our work advances literature in multiple ways:
\textbf{(1)}~We deliver a principled and effective approach to generalize \emph{any} first-order attribution method to quantify second-order interactions via the Shapley value.
Meta-attributions unite removal-based perturbations with gradient- and attention-based mechanisms to explain transformers.
\textbf{(2)}~We propose a \emph{hierarchical} decomposition of attributions via \emph{directional} interactions, which provides a practical way to navigate their complexity.
We show that, surprisingly, Shapley interactions and integrated Hessians implicitly do this already~(Theorem~\ref{thm:efficiency_existing}).
\textbf{(3)}~The \metagamex and meta-attributions define a novel \emph{directional} interaction index, and we prove this serves as a directional extension of existing interaction indices~(Theorem~\ref{thm:metasv}).
\textbf{(4)}~Finally, we demonstrate that meta-attributions offer pragmatic insights when applied to interpret large language models, vision--language encoders, and multimodal diffusion transformers.

\begin{figure}[t]
    \centering
    \includegraphics[width=0.91\linewidth]{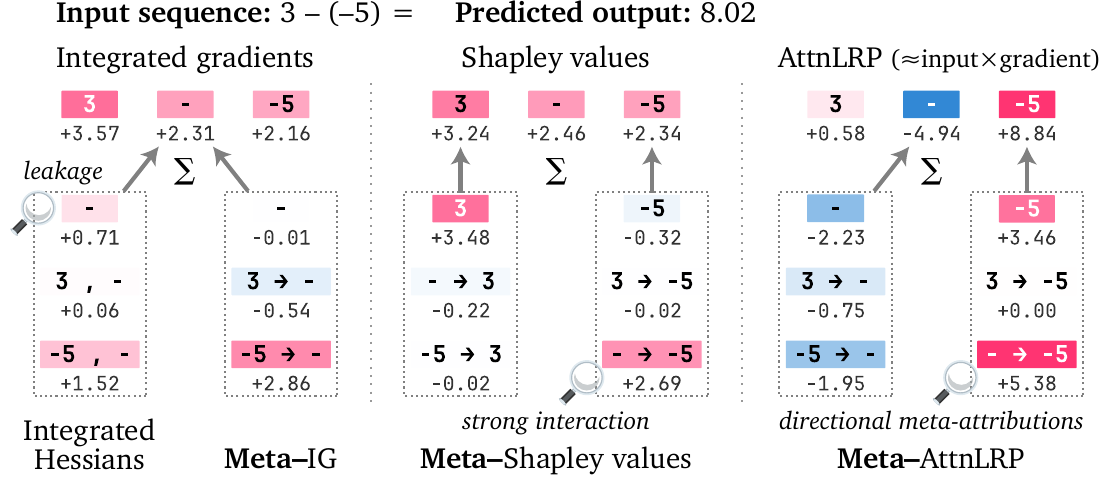}
    \caption{\textbf{Complementary interpretations of a simple transformer solving integer addition.}
    First-order attributions implicitly capture interactions between token pairs.
    Integrated Hessians highlight the undirected interaction \abox[27.40]{myred}{black}{\texttt{-5 , -}}, but exhibit information leakage for \abox[12.86]{myred}{black}{\texttt{-}}, which should be zero without knowing the subsequent number.
    AttnLRP, by contrast, assigns opposite-signed local attributions to \abox[90]{myblue}{white}{\texttt{-}} and~\abox[90]{myred}{white}{\texttt{-5}}.
    Our proposed meta-attributions assign individual effects correctly (e.g. \abox[62.58]{myred}{white}{\texttt{3}}, \abox[0.24]{myblue}{black}{\texttt{-}}) and quantify the directed effect of each input token on the first-order attribution of another token.
    Meta-Shapley values reveal that the pure effect of \abox[7.25]{myblue}{black}{\texttt{-5}} is small, and the attribution value is mostly determined by the interaction \abox[48.37]{myred}{black}{\texttt{- → -5}}.
    Meta-AttnLRP explains AttnLRP's opposite-signed attributions directionally: a minus contributes positively to the attribution of \abox[90]{myred}{white}{\texttt{-5}} when \texttt{-5} is present \abox[90]{myred}{white}{\texttt{- → -5}} (\texttt{+5.38}); conversely, removing
    \texttt{-5} weakens the negative signal of the minus \abox[43.90]{myblue}{black}{\texttt{-5 → -}} (\texttt{-1.95}), thereby offsetting the first-order attribution of \abox[90]{myblue}{white}{\texttt{-}} (\texttt{-4.94}).
    For details, see Appendix~\ref{app:illustrative_example}.
    }
    \label{fig:figure1}
\end{figure}

\section{Theoretical Background}
\label{sec:background}

We consider a general feature/token attribution method $\phi_i(f)$ that quantifies the influence of feature $i \in [d] := \{1,\dots,d\}$ on a model predictor $f \colon \mathbb{R}^d \to \mathbb{R}$. 
While broadly applicable, we predominantly focus on local explanations at a fixed input $x \in \mathbb{R}^d$, denoted as $\phi_i(f, x)$. 
To evaluate $f$ on subsets, we define the masked predictor $f(S; x)$, e.g. via baseline imputation $f(S; x) \coloneqq f(x_S, b_{\bar{S}})$, which retains features $S \subseteq [d]$ and replaces the remainder $\bar{S} \coloneqq [d] \setminus S$ with a baseline $b \in \mathbb{R}^d$.
Analogously, we overload the attribution notation and write $\phi_i(S;f,x) \coloneqq \phi_i\bigl(f,(x_S,b_{\bar S})\bigr)$ for the attribution of feature $i$ evaluated at the input $x$ with $\bar S$ masked by the baseline $b$, so that $\phi_i([d];f,x) = \phi_i(f,x)$.
  
\subsection{Interaction Quantification}
To interpret a model $f$, first-order attribution methods allocate the prediction score to individual features/tokens. 
A foundational game-theoretic approach for it is the Shapley value \citep[SV,][]{shapley1953value}, which computes the average marginal contribution of player $i$ across all possible coalitions:
\begin{align*}
    \phi_i^{\text{SV}}(f,x) := \sum_{S \subseteq [d] \setminus \{i\}} \frac{1}{d \cdot \binom{d-1}{\vert S \vert}} \left[ f(S \cup \{i\};x) - f(S;x) \right].
\end{align*}
In neural networks, input gradients enable attribution. 
Gradient$\times$input (G$\times$I) scales the gradient by the input, while layer-wise relevance propagation \citep[LRP,][]{bach2015pixel} applies custom backpropagation rules, linked to it~\citep{ancona2018towards}. 
To satisfy axioms such as efficiency, integrated gradients \citep[IG,][]{sundararajan2017axiomatic} accumulates gradients along a straight-line path:
\begin{align*}
    \phi_i^{\text{IG}}(f,x) \coloneqq (x_i - b_i) \int_0^1 \frac{\partial f(b + \alpha(x-b))}{\partial x_i} \, d\alpha.
\end{align*}

\textbf{Serial attributions} construct interactions by recursively applying a first-order method to its own output: $\psi_{i,j}(f,x) \coloneqq \phi_j\big(\phi_i(f, \cdot), x\big)$.
Notable examples include integrated Hessians $\psi^{\text{IH}}$ \citep[IH,][]{janizek2021explaining} and the serial Shapley value $\psi^{\text{SV}}$ \citep{lundstrom2023unifying}.
However, serial methods inherently fail to cleanly separate pure individual effects from interaction effects.

\textbf{Example: Interaction separation.} 
\cref{tab:interaction_separation} illustrates for $f(x) = x_1 + x_1x_2^2$ how a faithful method must isolate the individual effect of $x_1$ from the interaction $x_1x_2^2$. 
However, both the serial Shapley value and integrated Hessians fail to do this, leaking interaction terms into the individual components.

The \textbf{M\"obius transform} \citep{rota1964foundations}, $m_S(f,x) \coloneqq \sum_{T \subseteq S}(-1)^{\vert S \vert - \vert T\vert}f(T;x)$ for $S \subseteq [d]$ is a principled measure for separating pure interactions \citep{bordt2023from}. 
Due to its complexity with $2^d$ components, research has pivoted toward order-$k$ set-based interaction indices.

\textbf{Shapley interactions} extend the Shapley value to subsets of size $k$. For $k=2$, $\psi_{\{i,j\}}(f,x)$ quantifies the joint influence of features $i$ and $j$ beyond their individual contributions. Notable variants include $n$-Shapley values \citep{lundberg2020local,bordt2023from}, an extension of the Shapley interaction index \citep{grabisch1999axiomatic}, the faithful Shapley interaction index \citep{tsai2023faith}, and the Shapley-Taylor interaction index \citep[STII,][]{sundararajan2020shapley} given by
\begin{align*}
 \psi^{\text{STII}}_{\{i,j\}}(f,x) \coloneqq \sum_{S \subseteq [d] \setminus \{i,j\}} \frac{2}{d \cdot \binom{d-1}{\vert S \vert}}\sum_{T \subseteq \{i,j\}} (-1)^{\vert T \vert} f(S \cup T;x).
\end{align*}
At $k=2$, $\psi^{\text{STII}}_{\{i,j\}}$ separates pure individual effects and projects higher-order interactions onto pairs.

\textbf{Sum of Powers} \citep[SOP,][]{lundstrom2023unifying} theoretically extends integrated gradients to resolve the separation problem. 
For pairwise interactions, SOP evaluates a Shapley value on integrated gradients within a reduced $(d-1)$-player game where player $i$ is fixed as present. 
For $i \in [d]$, it can be defined by $\psi^{\text{SOP}}_{i,i}(f,x) \coloneqq \psi^{\text{SOP}}_{\{i\}}(f,x) \coloneqq m_{\{i\}}(f,x)$ for individuals and for interactions with $j \neq i$ via
\begin{align*}
\psi_{i,j}^{\text{SOP}}(f,x) \coloneqq \sum_{S \subseteq [d]\setminus\{i,j\}} \frac{1}{(d-1) \cdot \binom{d-2}{\vert S \vert}} \big[ \phi_i^{\text{IG}}(S \cup \{i, j\};f, x) - \phi_i^{\text{IG}}(S \cup \{i\};f, x) \big],
\end{align*}
where $\phi^{\text{IG}}_i(S;f,x)$ evaluates $\phi^{\text{IG}}_i$ restricted to the features in $S$.
The set-based SOP is then defined using the symmetrization $\psi^{\text{SOP}}_{\{i,j\}}(f,x) \coloneqq \psi_{i,j}^{\text{SOP}}(f,x) + \psi_{j,i}^{\text{SOP}}(f,x)$. 
In \cref{sec:metagame}, we prove that SOP follows our \emph{meta}-attribution paradigm.

\subsection{Limitations of Existing Interaction Measures}

Order-$k$ interaction indices produce $\mathcal O(d^k)$ terms, becoming difficult to interpret as dimensionality or interaction order grows.
By collapsing interactions into undirected subsets $\{i,j\}$, set-based approaches obscure the hierarchy of how features influence one another's attributions.
Serial methods fail to separate individual effects from interactions, leaking interaction terms into diagonal components.
While Shapley interactions and SOP isolate these effects (cf. \cref{tab:interaction_separation}), the former overlooks model internals, and the latter has been studied exclusively for integrated gradients.
Methods that use input Hessians are limited by locality and architectural constraints; they rely on fine-grained patch- or token-level evaluations without a natural mechanism for grouping, vanish in standard ReLU networks~\citep{janizek2021explaining}, and incur prohibitive computational costs~\citep{pramanik2026hessianenhanced}.

Our \metagamex framework treats any first-order attribution method as a cooperative game to decompose its attributions into directional interactions.
This preserves the directional hierarchy~($\varphi_{j \to i}$) and disentangles pure individual effects from interactions, thereby avoiding the leakage of serial methods.
Orthogonal to our \metagamex construction are directional game-theoretic measures that do not decompose attribution values.
Namely, \citet{lu2023asymmetric} average second-order discrete derivatives into an asymmetric Shapley interaction (ASI, $\Psi$), and \citet{masoomi2022explanations} propose a bivariate Shapley value (BSV, $E^2$).
These are limited to removal-based attribution and thus cannot utilize internal gradients or attentions.
\cref{tab:interaction_separation} compares their analytical interaction allocations.

\begin{table}[t]
\centering
\caption{Analytical interactions $\varphi_{j \to i}$ for $f(x) = x_1 + I$, where $I \coloneqq x_1x_2^2$. 
Unlike serial approaches, meta-attributions isolate the individual effect $x_1$ ($\varphi_{i \to i}$) and the interaction $I$ ($\varphi_{j \to i}$). 
Furthermore, meta-attributions capture the directional hierarchy and sum to the first-order attribution $\phi_i$, whereas set-based interactions collapse effects into undirected subsets.}
\label{tab:interaction_separation}
\vspace{0.5em}
\small
\setlength{\tabcolsep}{4.5pt}
\begin{tabular}{ll cc cc cc c}
\toprule
\multirow{2.5}{*}{\textbf{Type}} & \multirow{2.5}{*}{\textbf{Method}} & \multicolumn{2}{c}{\textbf{Individual}} & \multicolumn{2}{c}{\textbf{Interaction}} & \multicolumn{2}{c}{\textbf{First-order}} & \multirow{2.5}{*}{\textbf{Sum}}\\
\cmidrule(lr){3-4} \cmidrule(lr){5-6} \cmidrule(lr){7-8}
& & $\varphi_{1 \to 1}$ & $\varphi_{2 \to 2}$ & $\varphi_{2 \to 1}$ & $\varphi_{1 \to 2}$ & $\phi_1$ & $\phi_2$ \\ 
\cdashline{1-9}
\addlinespace[3pt]
 \multirow{2}{*}{Serial} & SV ($\psi^{\text{SV}}$) & $x_1 + \nicefrac{I}{4}$ & $\nicefrac{I}{4}$ & $\nicefrac{I}{4}$ & $\nicefrac{I}{4}$ & $x_1 + \nicefrac{I}{2}$ & $\nicefrac{I}{2}$ & $x_1+I$
\\
 & IH ($\psi^{\text{IH}}$) & $x_1 + \nicefrac{I}{9}$ & $\nicefrac{4I}{9}$ & $\nicefrac{2I}{9}$ & $\nicefrac{2I}{9}$ & $x_1 + \nicefrac{I}{3}$ & $\nicefrac{2I}{3}$ & $x_1+I$\\
\addlinespace[3pt]
\cdashline{1-9}
\addlinespace[3pt]
\multirow{3}{*}{Set-Based} & M\"obius ($m$) & $x_1$ & $0$ & \multicolumn{2}{c}{$I$} & $x_1$ & $0$ &$x_1$ \\
& Shapley ($\psi^{\text{STII}}$) & $x_1$ & $0$ & \multicolumn{2}{c}{$I$} & $x_1 + \nicefrac{I}{2}$ & $\nicefrac{I}{2}$ &$x_1+I$ 
\\
& SOP ($\psi^{\text{SOP}}$) & $x_1$ & $0$ & \multicolumn{2}{c}{$I$} & $x_1+\nicefrac{I}{3}$ & $\nicefrac{2I}{3}$ & $x_1 + I$
\\
\addlinespace[3pt]
\cdashline{1-9}
\addlinespace[3pt]
\multirow{2}{*}{Directional} & ASI ($\Psi$) & --- & --- & $\nicefrac{I}{2}$ & $\nicefrac{I}{2}$ & --- & --- & ---
\\
& BSV ($E^2$) & $x_1 + \nicefrac{I}{2}$ & $\nicefrac{x_1}{2} + \nicefrac{I}{2}$ & $\nicefrac{I}{2}$ & $\nicefrac{x_1}{2} + \nicefrac{I}{2}$ & $x_1 + I$ & $x_1 + I$ & $2x_1 + 2I$
\\
\addlinespace[3pt]
\cdashline{1-9}
\addlinespace[3pt]
\multirow{3}{*}{Ours} & Meta-G$\times$I & $x_1$ & $0$ & $I$ & $2 I$ & $x_1 + I$ & $2 I$ & $x_1+3I$ 
\\
& Meta-IG & $x_1$ & $0$ & $\nicefrac{I}{3}$ & $\nicefrac{2I}{3}$ & $x_1 + \nicefrac{I}{3}$ & $\nicefrac{2I}{3}$ & $x_1+I$ 
\\
 & Meta-SV & $x_1$ & $0$ & $\nicefrac{I}{2}$ & $\nicefrac{I}{2}$ & $x_1 + \nicefrac{I}{2}$ & $\nicefrac{I}{2}$ & $x_1 + I$\\
\bottomrule
\end{tabular}
\end{table}

\section{The Metagame of Interpretability and Meta-Attributions}
\label{sec:metagame}

We introduce \emph{meta}-attributions $\varphi_{j \to i}$, which can be read naturally as the influence of feature $j$ on the attribution of feature $i$, as a principled framework capable of decomposing \emph{any} first-order attribution method $\phi_i$ (e.g., Meta-IG decomposes $\phi_i^{\text{IG}}$). 
To compute these, we conceptualize the \metagamex, a cooperative game evaluating $\phi_i$, and apply the Shapley value to extract interactions.

\subsection{Hierarchical Interaction Decomposition of Attributions}\label{subsec:hierarchical_efficiency}
Instead of examining all $\mathcal O(d^2)$ interactions of order $2$ at once, we shift to a \emph{hierarchical} approach.
\begin{definition}
An interaction $\varphi_{j \to i}(f,x)$ is \textbf{hierarchically efficient} with respect to $\phi_i(f,x)$ if
\begin{align*}
    \phi_i(f,x) = \sum_{j \in [d]} \varphi_{j \to i}(f,x) = \varphi_{i \to i}(f,x) + \sum_{j \neq i} \varphi_{j \to i}(f,x).
\end{align*}
\end{definition}
Hierarchical efficiency has been an overlooked principle, and it is not specific to the Shapley value.
\citet{hausken2001value} already decompose the Shapley value in this sense, and we prove in \cref{app:hauskenmohr} that their construction is the serial Shapley value.
However, we propose that hierarchical efficiency is a desirable property of an interaction measure as it allows us to quantify second-order effects of any attribution mechanism.
It ensures that the original first-order attribution is exactly preserved as a marginal sum of its pure individual effect $\varphi_{i \to i}(f,x)$ and its interactions with other features $\varphi_{j \to i}(f,x)$. 
Surprisingly, while not originally framed this way, we show that several existing interaction indices already implicitly perform such hierarchical decompositions.
\begin{theorem}\label{thm:efficiency_existing}
   Existing interactions are hierarchically efficient, i.e. $\phi_i = \sum_{j \in [d]}\varphi_{j\to i}$:
   Shapley interactions decompose $\phi^{\text{SV}}$ with $\varphi_{j \to i} := \nicefrac{\psi_{\{i,j\}}}{\vert\{i,j\}\vert}$, whereas SOP and serial methods decompose $\phi^{\text{IG}}$ and their own first-order attributions with $\varphi_{j \to i} := \psi_{i,j}$. Proofs are in Appendix~\ref{app:proofs}.
\end{theorem}
For serial methods, hierarchical efficiency follows from the first-order efficiency; the outer attribution step (feature $j$) naturally decomposes the inner attribution (feature $i$). 
In set-based interactions, the factor of $\nicefrac{1}{\vert \{i,j\}\vert}$ accounts for the neglected directionality ($j \to i$ vs. $i \to j$).
However, these interaction measures are still symmetric and \emph{non-directional} due to their construction.

\subsection{Directional Interactions via Meta-Attributions}
We introduce \emph{directional} interactions for any first-order attribution method $\phi_i(f,x)$. 
One shortcoming of serial methods, whether via Shapley masking or by integrating Hessians, is that both features $x_i$ and $x_j$ are varied jointly. 
This joint variation inextricably entangles individual and pairwise effects.  
We therefore fix the feature value $x_i$ and evaluate the marginal contribution of $x_j$ directly with respect to the first-order attribution $\phi_i(f,x)$. 
To rigorously account for higher-order dependencies and fairly distribute the interactions, we compute the Shapley value of feature $j$ over the $d-1$ remaining features, treating $\phi_i$ evaluated at the fixed $x_i$ as the underlying~\metagamex.
\begin{definition}\label{def:metagame}
For a target feature $i \in [d]$ and an attribution method $\phi_i(f,x)$, the \textbf{\metagamex} is the $(d-1)$-player cooperative game $\nu_i \colon 2^{[d] \setminus \{i\}} \to \mathbb{R}$ with players $[d] \setminus \{i\}$ and a game function $\nu_i(S) \coloneqq \phi_i(S \cup \{i\};f,x)$, i.e. the attribution of $i$ when features outside of $S\cup\{i\}$ are masked.
\end{definition}
\begin{definition}\label{def:meta_attribution}
The \textbf{meta-attribution} $\varphi_{j \to i}(f,x)$, representing the influence of feature $j \neq i$ on the attribution $\phi_i(f,x)$ of feature $i$, is defined as the Shapley value of player $j$ in the \metagamex $\nu_i$:
$$\varphi_{j \to i}(f,x) := \sum_{S \subseteq [d]\setminus\{i,j\}} \frac{1}{(d-1) \cdot \binom{d-2}{\vert S \vert}} \big[ \phi_i(S \cup \{i, j\};f,x) - \phi_i(S \cup \{i\};f,x) \big].$$
The pure individual effect is the worth of the empty coalition, $\varphi_{i \to i}(f,x) \coloneqq \nu_i(\emptyset) = \phi_i(\{i\};f,x)$.
\end{definition}
\begin{corollary}\label{cor:directional}
Meta-attributions are hierarchically efficient, i.e.  
$\phi_i(f,x) = \sum_{j \in [d]} \varphi_{j \to i}(f,x)$.
\end{corollary}
This follows from the efficiency of the Shapley value, since the grand coalition of the \metagamex is worth $\nu_i([d] \setminus \{i\}) = \phi_i(f,x)$ and the empty coalition is worth $\varphi_{i \to i}(f,x)$.
Hence, although we use the Shapley value as the efficient allocation rule that relates meta-attributions to existing interaction indices, \metagamex is extensible to other efficient allocation methods as a core principle.

When the underlying first-order attribution is itself the Shapley value (i.e., $\phi_i := \phi_i^{\text{SV}}$), we refer to the resulting meta-attribution $\varphi_{j \to i}^{\text{Meta-SV}}$ as the \textbf{Meta-Shapley value}. 
Similarly, $\phi_i^{\text{IG}}$ or $\phi_i^{\text{G}\times\text{I}}$ yields $\varphi_{j \to i}^{\text{Meta-IG}}$ and $\varphi_{j \to i}^{\text{Meta-G}\times\text{I}}$, respectively.
We relate these methods to set-based interactions in Theorem~\ref{thm:metasv}, which establishes the Meta-Shapley value and Meta-IG as \emph{directional} variants of STII and SOP, respectively. 
Crucially, the \metagamex framework extends this principled directional decomposition to arbitrary first-order attribution methods, e.g. those that rely on gradients and attention in transformers.
\begin{theorem}\label{thm:metasv}
The Meta-Shapley value and Meta-IG act as directional variants of the STII and SOP:
\begin{align*}
    \psi^{\text{STII}}_{\{i\}}(f,x) &= \varphi_{i \to i}^{\text{Meta-SV}}(f,x), & \psi^{\text{STII}}_{\{i,j\}}(f,x) &= \varphi_{j \to i}^{\text{Meta-SV}}(f,x) + \varphi_{i \to j}^{\text{Meta-SV}}(f,x), \\
    \psi^{\text{SOP}}_{\{i\}}(f,x) &= \varphi_{i \to i}^{\text{Meta-IG}}(f,x), & \psi^{\text{SOP}}_{\{i,j\}}(f,x) &= \varphi_{j \to i}^{\text{Meta-IG}}(f,x) + \varphi_{i \to j}^{\text{Meta-IG}}(f,x).
\end{align*}
\end{theorem}

\textbf{Example: Interaction separation.} 
Continuing with the example $f(x)=x_1 + x_1x_2^2$ (\cref{tab:interaction_separation}), we find that directional meta-attributions isolate the pure individual effect $x_1$ into the diagonal $\varphi_{1 \to 1}$ and correctly yield $\varphi_{2 \to 2} = 0$. 
The interaction term $I \coloneqq x_1x_2^2$ is strictly restricted to the off-diagonal terms ($\varphi_{2 \to 1}$ and $\varphi_{1 \to 2}$), completely avoiding the spurious leakage seen in serial methods. 
The exact allocation of this interaction depends on the underlying first-order method. 
Notably, our framework exposes the inherent \emph{directionality} of feature interactions: while the Meta-SV yields symmetric interactions by design ($\varphi_{2 \to 1} \equiv \varphi_{1 \to 2}$), gradient-based methods like Meta-IG and Meta-G$\times$I capture asymmetric, directional influences ($\varphi_{2 \to 1} \neq \varphi_{1 \to 2}$). 
Regardless of symmetry, directional variants separate individual from interaction effects and sum to their first-order attributions.

\subsection{Considerations for Applying the \Metagamex Principles to Interpret Transformers}

While our theoretical considerations focus on the foundational attribution methods, the general concepts and intuitions behind \metagamex are also relevant to the modern interpretability approaches for transformers.
We effectively demonstrate this through three settings of increasing abstractness.

\textbf{Meta-AttnLRP.}
AttnLRP~\citep{achtibat2024attnlrp} is a backpropagation-based, attention-aware method to attribute single input tokens into transformer predictions. 
While it has shown state-of-the-art faithfulness performance in language and vision tasks, it is unable to capture token--token synergies. 
Meta-AttnLRP complements its interpretation with directional second-order interactions by computing a Shapley value of the exact \metagamex defined as: players $\coloneqq$ input tokens, game function $\coloneqq$ AttnLRP (cf. Definition~\ref{def:metagame}).
Specifically for language models, we simply aggregate token attributions and interactions across the generated multi-token output sequence as detailed in~Appendix~\ref{app:attnlrp}.

\textbf{Meta-Grad-ECLIP.}
Grad-ECLIP~\citep{chenyang2024gradeclip} is a gradient-based method to measure the influence of image and text tokens on the embeddings of vision--language encoders.
Similarly to AttnLRP, it cannot quantify the synergy between the image and text modalities.
Our proposed Meta-Grad-ECLIP quantifies directional cross-modal interactions by computing Shapley values of text tokens in the attribution of image tokens.
\Metagamex thus becomes: players $\coloneqq$ text tokens, game function $\coloneqq$ Grad-ECLIP of image tokens (cf. Figure~\ref{fig:metagradeclip_example}).
Here, the target $i$ is an image patch rather than a text token, so the \metagamex has text tokens as players, and the game function $\nu_i(S)$ becomes the attribution of patch $i$ when text tokens outside of $S$ are masked.
Extending this idea, one can also compute a bi-token Shapley interaction of the \metagamex, which can be treated as measuring synergy between a \emph{token pair} and the image, e.g. \texttt{black--dog} and \texttt{hydr--ant} in Figure~\ref{fig:metagradeclip_example}.
Finally, \metagamex could also be implemented the other way around to quantify the second-order effects of image patches on the attribution of text tokens, which we do not investigate in this work.
As a minor contribution of this paper, we extend the attention- and gradient-based methods originally proposed for CLIP to the modern SigLIP-2 architecture~\citep{tschannen2025siglip2} as detailed in Appendix~\ref{app:gradesiglip}.

\begin{figure}[t]
    \centering
    \includegraphics[width=\linewidth]{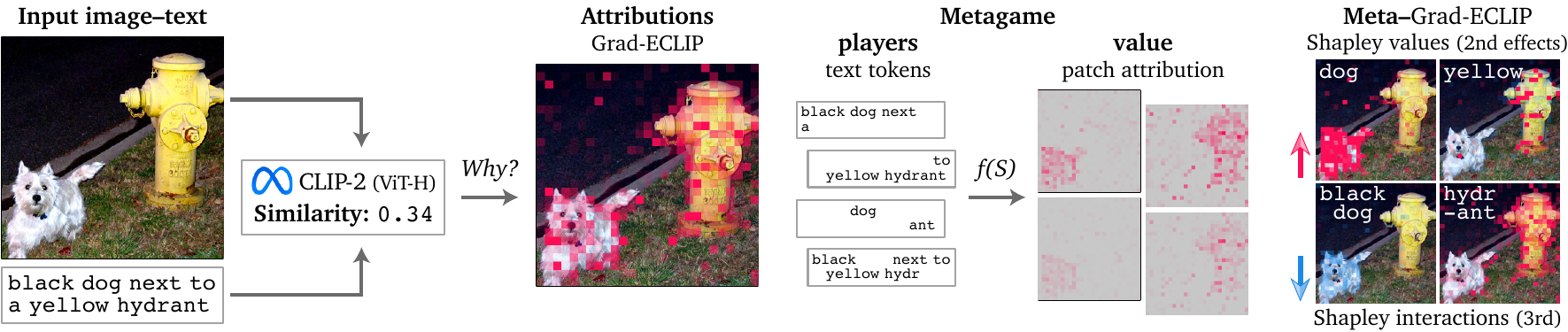}
    \caption{\textbf{\Metagamex quantifies gradient-based token interactions in vision--language encoders.} 
    Given a token attribution method (Grad-ECLIP) and a dual encoder (MetaCLIP-2), we compute meta-attributions from text token subsets and their corresponding visual patch attributions.
    First-order attributions quantify the effects that text tokens \texttt{dog} and \texttt{yellow} have on the similarity map (red, most similar).
    Directional meta-attributions quantify the interaction effects contributed by pairs of text tokens, such as \texttt{hydr-ant} and \texttt{black-dog} (blue, most dissimilar).
    See also App. Figure~\ref{fig:metagradeclip_example_extended}.
    }
    \label{fig:metagradeclip_example}
\end{figure}

\textbf{Meta-ConceptAttention.} 
ConceptAttention is not a standard attribution method; it allows attributing concepts (text tokens) to image patches generated by text-to-image diffusion transformers.
However, it is limited by producing highly context-dependent interpretations, as discussed in Section~5.5 and Figure~15 of~\citet{helbling2025conceptattention}.
We aim to alleviate context dependence and interpret the second-order effects of concepts.
As the most abstract application of our framework, departing from \cref{def:metagame}, Meta-ConceptAttention computes a Shapley value of the \metagamex defined as: players $\coloneqq$ concepts, game function $\coloneqq$ ConceptAttention.
Interestingly, this application also illustrates a case where including the target concept $i$ as a player can be useful, since the Shapley value of $i$ then averages its attribution across all contexts of the other concepts.
Crucially, our implementation requires only a single model forward pass, just like ConceptAttention, and computes marginal contributions using the model's cached embeddings, making it viably efficient. 
We provide further methodological and implementational details in Appendix~\ref{app:conceptattention}.

\textbf{Scaling the computation.}
For small \textsc{metagames} ($d < 20$), we compute the Shapley value exactly, jointly obtaining $\varphi_{j \to i}$ for all targets $i$ in one vectorized pass.
For larger ones, we recommend SOTA approximators from the \texttt{shapiq} library~\citep{muschalik2024shapiq}: Monte Carlo sampling amortizes well with moderate $d$ and many targets, while regression-based approximators~\citep{witter2026regression} better focus the budget when $d$ is large but only specific targets matter, at the cost of one fit per target.

\section{Applications and Experimental Results}
\label{sec:experiments}

Here, we demonstrate that the \metagamex offers useful insights: 
(i) when quantifying token interactions in language models,
(ii) when explaining similarity in vision--language encoders, and
(iii) when interpreting concepts in multimodal diffusion transformers.
See Appendix~\ref{app:experimental_setup} for details regarding the experimental setup, implementations, and hyperparameters.

\begin{figure}[t]
    \centering
    \includegraphics[width=\linewidth]{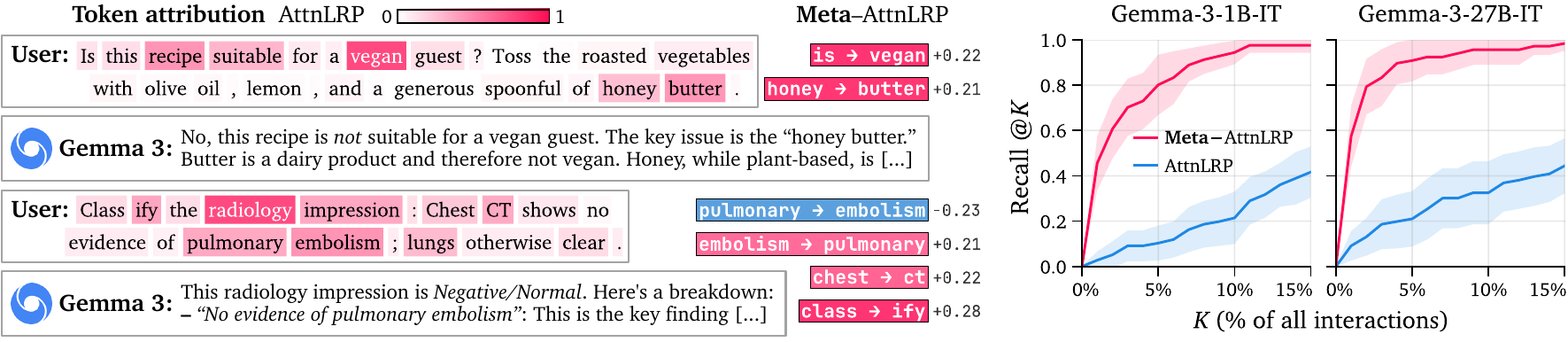}
    \caption{
    \textbf{\Metagamex quantifies token interactions in instruction-tuned large language models.}
    We compute Meta-AttnLRP as Shapley values from text tokens into AttnLRP token attributions of the Gemma language model's generated output, highlighting directional second-order effects.
    We also measure the recall of detecting \mbox{human-labeled} interactions (e.g. word connotations, negation) on a sample of prompts spanning various contexts.
    For results on other models, see App.~Figure~\ref{fig:metaattnlrp_extended}.
    }
    \label{fig:metaattnlrp}
\end{figure}

\subsection{Quantifying Token Interactions in Language Models}

We interpret the Gemma 3 language model~\citep{gemmateam2025gemma3} via AttnLRP~\citep{achtibat2024attnlrp}, although in principle \metagamex is general and any token attribution method can be used instead.
Figure~\ref{fig:metaattnlrp} illustrates two examples of how important each token in the user prompt is to the model's generated response, averaged over up to 100 output tokens.
In the first case, the first-order attributions rank \texttt{recipe}, \texttt{vegan}, and \texttt{butter} as the most influential, while our proposed Meta-AttnLRP complements the interpretation with interactions between \texttt{is}→\texttt{vegan} and \texttt{honey}→\texttt{butter} as key to the model's comprehension. 
In another example, AttnLRP emphasizes \texttt{radiology}, \texttt{embolism}, and \texttt{ct}, while Meta-AttnLRP highlights the importance of \texttt{class}→\texttt{ify} and \texttt{chest}→\texttt{ct}.
Crucially, we detect that the model pays more attention to \texttt{pulmonary} when \texttt{embolism} is present, whereas removing \texttt{pulmonary} would increase the importance of \texttt{embolism} in the generated output.
We report the increasing rate at which both approaches are able to capture such concepts using various human-labeled prompts, each containing 2--4 interacting token pairs, in Figure~\ref{fig:metaattnlrp} (right).
We provide additional examples where \metagamex complements AttnLRP attributions in Appendix~\ref{app:additional_experimental_results}.

\subsection{Explaining Similarity in Vision--Language Encoders}
\label{sec:experiments_clip}

We demonstrate the broader applicability of \metagamex to improve the explainability of vision--language encoders such as CLIP~\citep{radford2021learning}, SigLIP-2~\citep{tschannen2025siglip2}, and MetaCLIP-2~\citep{chuang2025metaclip2}.
We compare three popular saliency map methods for CLIP, namely generic attention~\citep{chefer2021generic}, MaskCLIP~\citep{zhou2022maskclip}, and Grad-ECLIP~\citep{chenyang2024gradeclip}, with their meta-attribution counterparts.
We benchmark the performance of these algorithms quantitatively using the pointing interaction recognition metric proposed in~\citep[FIxLIP,][]{baniecki2025explaining}, with a random method scoring 0.25.
It combines four labeled images into a grid and scores the share of interaction mass that an explanation places on the correct pairs, i.e. between each object word and the patches of the image showing that object.
Figure~\ref{fig:pointing_game} shows an example of evaluating Meta-Grad-ECLIP interactions explaining MetaCLIP-2 on the \texttt{fish-koala-balloon-laptop} pointing game.
Table~\ref{tab:metagradeclip_pointing_game} reports the state-of-the-art performance of meta-attributions in consistently improving the interaction recognition on ImageNet-1k across various models and methods. 
Meta-Attention outperforms Meta-Grad-ECLIP on the overall largest and best-performing MetaCLIP-2 model.
Interestingly, our proposed Meta-MaskSigLIP approach is overall the most faithful in interpreting SigLIP-2, a result confirmed across two model patch sizes (cf. App. Table~\ref{tab:metagradeclip_pointing_game_extended}).
Meta-attributions often outperform even the second-order baselines based on Shapley interactions and gradients in this task.
Beyond faithfully attributing bi-token interactions, computing Shapley interactions for the \metagamex of these attributions quantifies the tri-token influence on the model's prediction, as illustrated in Appendix Figures~\ref{fig:metagradeclip_example_extended}~\&~\ref{fig:metagradeclip_example_negative}.

\begin{figure}[t]
    \centering
    \includegraphics[width=0.245\linewidth]{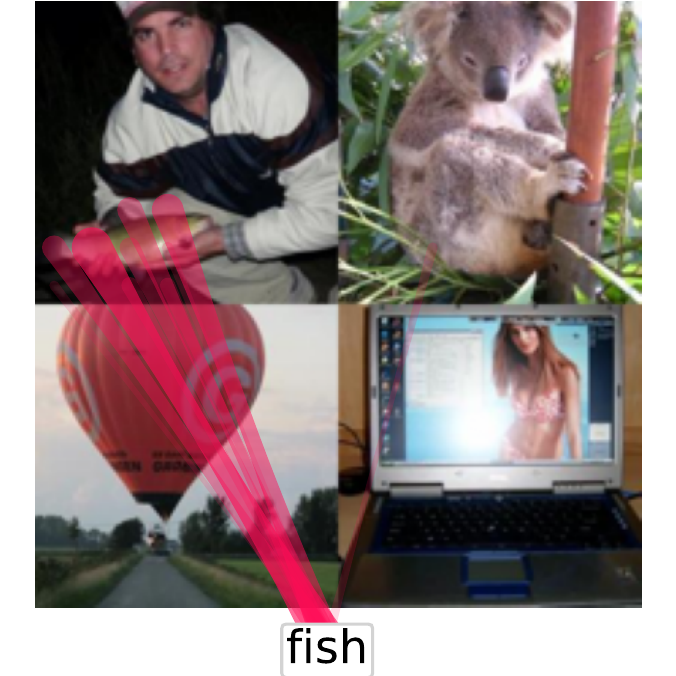}
    \includegraphics[width=0.245\linewidth]{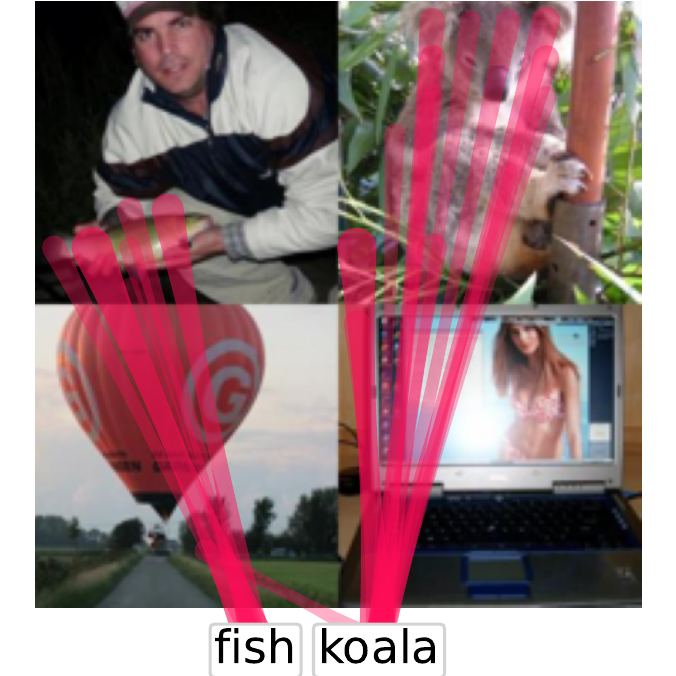}
    \includegraphics[width=0.245\linewidth]{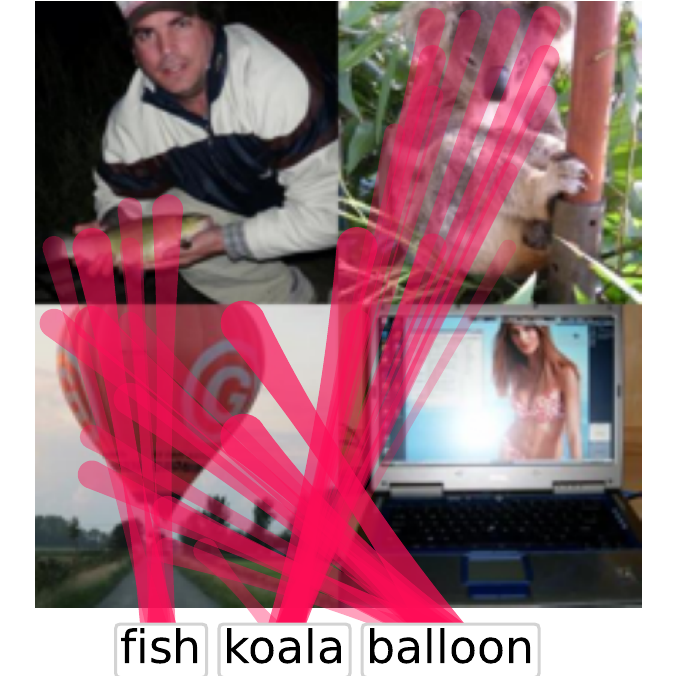}
    \includegraphics[width=0.245\linewidth]{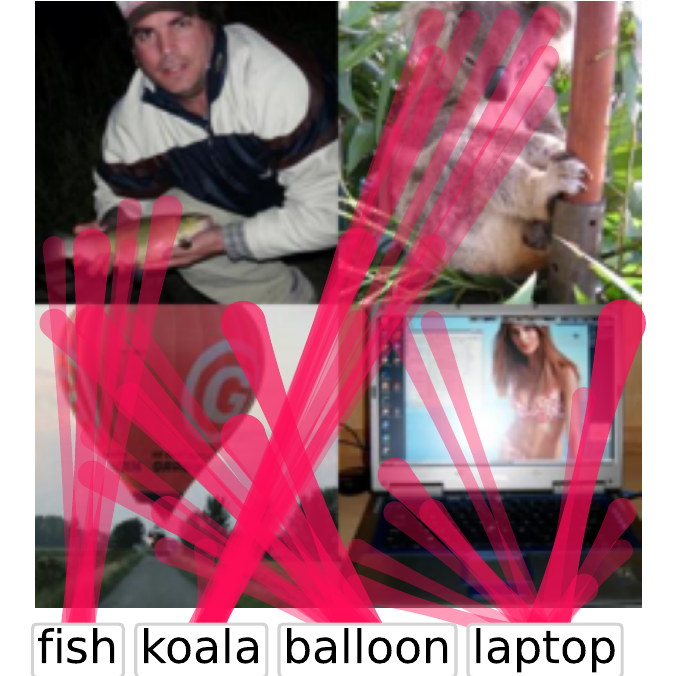}
    \caption{Example of Meta-Grad-ECLIP explaining the MetaCLIP-2 encoder on a single input from a single pointing game, which is aggregated to measure the interaction recognition evaluation metric.}
    \label{fig:pointing_game}
\end{figure}

  \begin{table}[t]
      \caption{
        \textbf{\textsc{Metagame} lifts first-order explanations of vision--language encoders to second order.}
        On ImageNet-1k, meta-attributions correctly recover the cross-modal interactions that first-order methods miss across models of different architectures and sizes.
        The second-order baselines {\color{gray}\textbf{in gray}} take much longer to compute.
        Extended results for SigLIP-2 are reported in Appendix Table~\ref{tab:metagradeclip_pointing_game_extended}.
      }
      \label{tab:metagradeclip_pointing_game}
      \vspace{0.5em}
      \centering
      \small
      \begin{tblr}{
          colspec = {llcccc},
          cell{4,6,8,12,14}{2-6} = {bg=black!10},
          cell{3}{1} = {r=8}{l},
          cell{11}{1} = {r=6}{l},
          hline{3, 11} = {dashed, 0.5pt},
          row{9,10,15,16} = {fg=gray}
      }
          \toprule
          \SetCell[r=2]{l} \textbf{Model (Size)} & \SetCell[r=2]{l} \textbf{Explanation Method} & \SetCell[c=4]{c} \textbf{Interaction Recognition}
  $(\uparrow)$ & & & \\
           & & 1 object & 2 objects & 3 objects & 4 objects \\
          {CLIP \\ (ViT-B/16)}
          & Attention~\citep{chefer2021generic} & $.77_{\pm.01}$ & $.46_{\pm.01}$ & $.32_{\pm.01}$ & $.25_{\pm.00}$ \\
          & \hspace{2pt}\includegraphics[width=6pt]{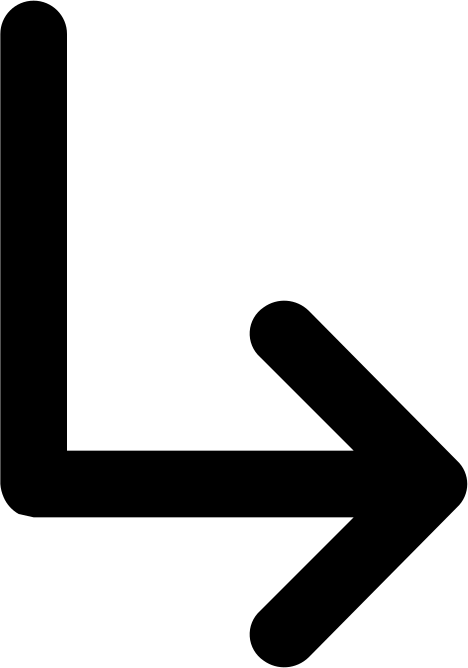} Meta-Attention & $.84_{\pm.01}$ & $.87_{\pm.01}$ & $\bm{.89_{\pm.01}}$ &
  $\bm{.90_{\pm.01}}$ \\
          & MaskCLIP~\citep{zhou2022maskclip} & $.30_{\pm.01}$ & $.28_{\pm.01}$ & $.27_{\pm.01}$ & $.25_{\pm.00}$ \\
          & \hspace{2pt}\includegraphics[width=6pt]{figures/arrow.png} Meta-MaskCLIP & $\bm{.86_{\pm.01}}$ & $\bm{.88_{\pm.01}}$ & $\bm{.89_{\pm.01}}$
  & $.89_{\pm.01}$ \\
          & Grad-ECLIP~\citep{chenyang2024gradeclip} & $.71_{\pm.01}$ & $.43_{\pm.01}$ & $.31_{\pm.01}$ & $.25_{\pm.00}$ \\
          & \hspace{2pt}\includegraphics[width=6pt]{figures/arrow.png} Meta-Grad-ECLIP & $.76_{\pm.01}$ & $.79_{\pm.01}$ & $.82_{\pm.01}$ &
  $.84_{\pm.01}$ \\
          & FIxLIP~\citep{baniecki2025explaining} & $.76_{\pm.01}$ & $.77_{\pm.01}$ & $.78_{\pm.01}$ & $.79_{\pm.01}$ \\
          & exCLIP~\citep{moeller2025explaining} & $.71_{\pm.01}$ & $.80_{\pm.01}$ & $.84_{\pm.01}$ & $.86_{\pm.01}$ \\
          {MetaCLIP-2 \\ (ViT-H/14)}
          & Attention~\citep{chefer2021generic} & $.84_{\pm.01}$ & $.45_{\pm.01}$ & $.32_{\pm.01}$ & $.25_{\pm.00}$ \\
          & \hspace{2pt}\includegraphics[width=6pt]{figures/arrow.png} Meta-Attention & $\bm{.85_{\pm.01}}$ & $\bm{.83_{\pm.01}}$ & $\bm{.85_{\pm.01}}$
  & $\bm{.86_{\pm.01}}$ \\
          & Grad-ECLIP~\citep{chenyang2024gradeclip} & $.76_{\pm.01}$ & $.43_{\pm.01}$ & $.31_{\pm.01}$ & $.25_{\pm.00}$ \\
          & \hspace{2pt}\includegraphics[width=6pt]{figures/arrow.png} Meta-Grad-ECLIP & $.79_{\pm.01}$ & $.76_{\pm.01}$ & $.78_{\pm.01}$ &
  $.80_{\pm.01}$ \\
          & FIxLIP~\citep{baniecki2025explaining} & $.70_{\pm.01}$ & $.66_{\pm.01}$ & $.73_{\pm.01}$ & $.75_{\pm.01}$ \\
          & exCLIP~\citep{moeller2025explaining} & $.83_{\pm.01}$ & $.81_{\pm.01}$ & $.83_{\pm.01}$ & $.86_{\pm.01}$ \\
          \bottomrule
      \end{tblr}
  \end{table}

\subsection{Interpreting Concepts in Multimodal Diffusion Transformers}

Finally, we show the potential of \metagamex to improve the interpretability of text-to-image diffusion transformers.
Meta-ConceptAttention provides two valuable utilities in this task: the Shapley value of ConceptAttention improves on its main limitation---context-dependence, while the~\emph{directional} meta-attribution explains the influence between pairs of concepts~(see Figure~\ref{fig:metaconceptattention_example}). 
Following \citet{helbling2025conceptattention}, we benchmark Meta-ConceptAttention on a proxy segmentation performance of the FLUX.1 [schnell] text-to-image model~\citep{blackforestlabs2024flux} on the Pascal VOC dataset~\citep{everingham2015pascal}, including single-object and multiclass images.
We extend the evaluation to the more challenging MS COCO benchmark~\citep{lin2014microsoft}.
Contrary to the original experimental setup, we evaluate segmentation performance when additional concepts are included in the context.
Table~\ref{tab:metaconceptattention} shows that Meta-ConceptAttention Pareto-dominates the original ConceptAttention across all metrics for both the single and multiple concept tasks, when simultaneously evaluating the method on all classes apparent in Pascal VOC and the same amount for MS COCO.
We further perform extensive ablations on the number of added concepts, including the artificial text prompt, not using cross-concept attention, and restricting the set of layers, as reported in Appendix Figure~\ref{fig:metaconceptattention_extended}.
We find that our Meta-Shapley-based interpretation correctly denoises attention maps, i.e. its performance stays the same or even increases as the number of additional concepts in-context increases.
Notably, if we did not use cross-concept attention, the performance gap would be even larger.

\begin{figure}
    \centering
    \includegraphics[width=\linewidth]{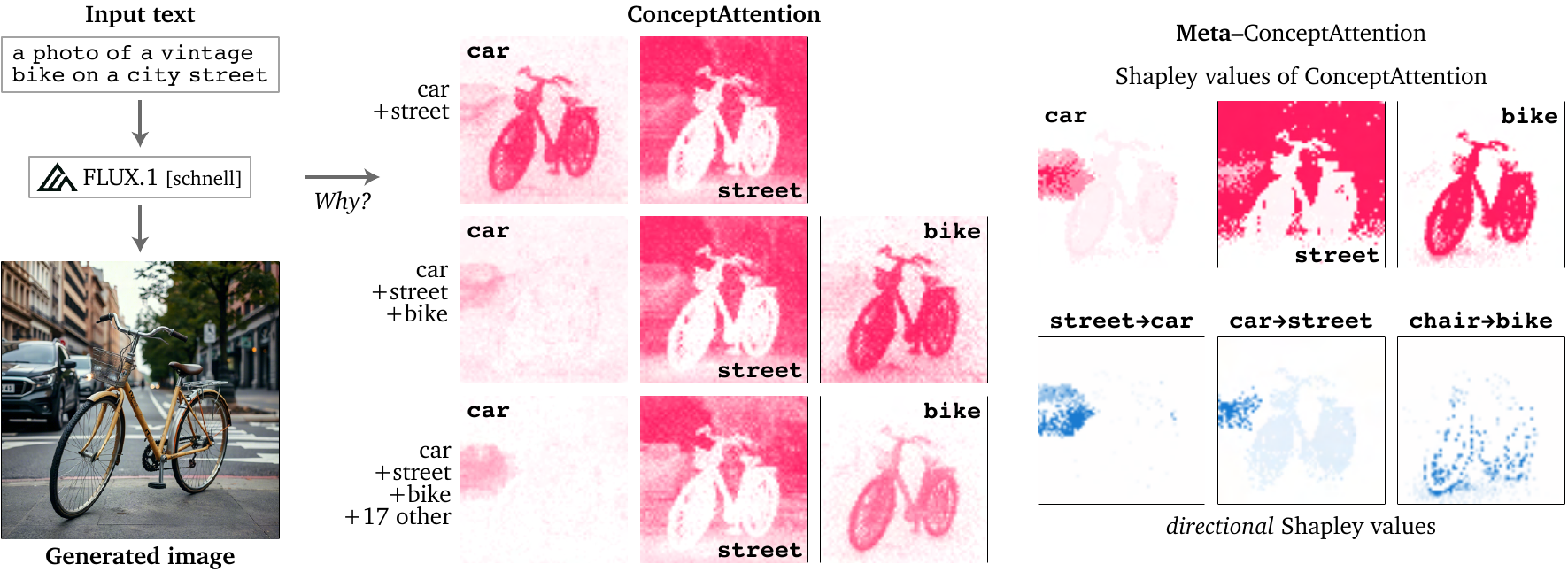}
    \caption{
    \textbf{\Metagamex quantifies concept interactions in multimodal diffusion transformers.} 
    Shapley values average attention across concept subsets, and interpret their directional dependencies.
    }
    \label{fig:metaconceptattention_example}
\end{figure}

\begin{table}[t]
    \caption{\Metagamex alleviates the context dependence in the ConceptAttention interpretability of the FLUX.1 [schnell] text-to-image model across Pascal VOC and MS COCO datasets.
    Crucially, both methods use only a single model forward pass.
    Several ablations are reported in App. Figure~\ref{fig:metaconceptattention_extended}.}
    \label{tab:metaconceptattention}
    \vspace{0.5em}
    \centering
    \small
    \begin{tblr}{
        colspec = {l cccc},
        cell{4}{1-5} = {bg=black!10}, 
        hline{3} = {dashed, 0.5pt},
    }
        \toprule
        \SetCell[r=2]{l} \textbf{Method} & \SetCell[c=2]{c} \textbf{Pascal VOC} (Acc/mIoU/mAP $\uparrow$) & & \SetCell[c=2]{c} \textbf{MS COCO} (Acc/mIoU/mAP $\uparrow$) & \\
        \cmidrule[dashed, lr]{2-3} \cmidrule[dashed, lr]{4-5}
        & Single concept & Multiple concepts & Single concept & Multiple concepts \\
        ConceptAttention & $.832/.645/.878$ & $.612/.476/.622$ & $.852/.656/.899$ & $.572/.350/.423$ \\
        \hspace{2pt}\includegraphics[width=6pt]{figures/arrow.png} \metagame & $\bm{.903}/\bm{.768}/\bm{.908}$ & $\bm{.803}/\bm{.553}/\bm{.663}$ & $\bm{.889}/\bm{.720}/\bm{.909}$ & $\bm{.706}/\bm{.390}/\bm{.439}$ \\
        \bottomrule
    \end{tblr}
\end{table}

\section{Related Work}\label{sec:related_work}

While \metagamex shares a naming affinity with the meta-evaluation problem~\citep{hedstrom2023meta}, which quantifies attribution quality, our framework focuses on second-order interaction effects.

\textbf{Backpropagation-based attribution.} 
Several gradient-based methods capture higher-order dynamics implicitly by aggregating attributions in a local neighborhood around the input, e.g. SmoothGrad~\citep{smilkov2017smoothgrad,zhou2025adaptgrad}, expected gradients~\citep{erion2021improving,baniecki2025efficient}, or through model randomization tests~\citep{adebayo2018sanity,binder2023shortcomings}. 
\citet{sikdar2021integrated} utilizes integrated directional gradients to explicitly quantify group interactions, while \citet{eberle2022building} extend layer-wise relevance propagation to quantify interactions in deep similarity networks via second-order Taylor decomposition. 
Specifically in data attribution, \citet{wang2025data} leverage gradient-Hessian-gradient products to capture interactions between training samples.

\textbf{Removal-based attribution.} 
Game-theoretic attribution is continuously evolving to capture complex dependencies, from asymmetric Shapley values for causal priors~\citep{frye2021asymmetric}, to $n$-Shapley values recovering a functional decomposition of the model~\citep{bordt2023from}, and unified frameworks for higher-order interactions~\citep{fumagalli2025unifying}. 
Other advances include discrete representations of mixed partial derivatives~\citep{tsang2020how} and mitigations for non-convex games~\citep{chang2025rethinking}. 
However, \metagamex shifts the paradigm: rather than evaluating the model output directly, it treats the first-order attribution itself as the value function of a cooperative game to extract directional interactions. 
Finally, concurrent efforts by \citet{kang2025spex} and \citet{butler2025proxyspex} focus on scaling token interactions in large language models via Fourier expansions.

\textbf{Applications.} 
Despite ongoing skepticism regarding the specific properties of some feature attribution methods~\citep{bilodeau2024impossibility,bohle2024bcos}, they remain crucial for analyzing large vision models~\citep{achtibat2023attribution,fel2023holistic}, e.g. in medical imaging~\citep{chrabaszcz2025aggregated}.
\citet{komorowski2026attribution} leveraged AttnLRP to guide decoding and improve instruction-following in language models like Gemma~3.
Explicit second-order methods, such as integrated Jacobians, successfully explain CLIP~\citep{moeller2025explaining} and LLMs~\citep{pramanik2026hessianenhanced}. 
Similarly, Shapley and Banzhaf interactions provide actionable insights in video-language training~\citep{jin2024hierarchical}, hyperparameter optimization~\citep{wever2026hypershap}, and survival analysis~\citep{langbein2026functional}.

\textbf{Attention.} 
The debate over whether attention faithfully explains transformers remains central, particularly regarding how to aggregate these values into an explanation while accounting for other network components~\citep{wiegreffe2019attention,bibal2022attention}. 
Growing evidence highlights its limitations across vision~\citep{komorowski2023towards}, language~\citep{achtibat2024attnlrp}, and multimodal domains~\citep{chenyang2024gradeclip}. 
Nonetheless, attention remains a practical tool for locating concepts in diffusion transformers~\citep{helbling2025conceptattention,erel2025attention}. 
To overcome the inherent limitations of standard transformer architectures~\citep{dziri2023faith,lin2025zebralogic}, third-order attention between input tokens is increasingly explored~\citep{sanford2023representational,kozachinskiy2025strassen}.

\section{Conclusion}
\label{sec:conclusion}

In this work, we introduced the \metagamex and meta-attributions as a mathematically principled interpretability framework that complements existing methods for quantifying second-order interactions. 
We demonstrated its versatility and impact across three areas of AI: analyzing token interactions in instruction-tuned language models, explaining vision-language encoders, and interpreting concepts in multimodal diffusion transformers. 
Ultimately, the foundational constructs behind \metagamex offer a promising path toward more credible machine learning systems through future empirical research.

\textbf{Limitations and future work.} 
Clearly, meta-attributions apply to first-order attributions that can be evaluated on masked inputs.
Our theory is limited by assuming baseline masking/imputation, the most principled approach to removing tokens/features from machine learning model inputs.
Future work may extend the \metagamex to other removal schemes and higher-order interactions as we elaborate in Appendix~\ref{app:metagame_extensions}. 
We ablate the masking strategy for CLIP meta-attributions in Table~\ref{tab:masking_sensitivity}, finding that the results are robust to this choice in our setup.
In the vision--language and diffusion experiments, we compute Shapley values for meta-attributions \emph{exactly}, at a computational cost comparable to that of the first-order baselines.
For longer language model prompts, we rely on coalition sampling, so \metagamex will only benefit from future improvements in Shapley value estimation.
We further acknowledge that our language model analysis is a systematic qualitative showcase of a few prompts; rigorous method comparisons would also require faithfulness measurements~\citep{jumelet2023feature}, such as token-pair insertion/deletion curves~\citep{baniecki2025explaining}.
Finally, quantifying higher-order effects \emph{inherently} increases both computational cost and cognitive burden for the user.

\textbf{Broader impact.}
We do not anticipate any negative societal impacts from this work, as it focuses on foundational interpretability techniques intended to improve our understanding of transformers. 

\begin{ack}
Hubert Baniecki was supported by the Foundation for Polish Science
(FNP), and the state budget within the Polish Ministry of Education and Science program ``Pearls of Science'' project number PN/01/0087/2022.
Fabian Fumagalli gratefully acknowledges funding by the Deutsche Forschungsgemeinschaft (DFG, German Research Foundation): TRR 318/3 2026 – 438445824.
Work on this project is financially supported by the Foundation for Polish Science (FNP) grant `Centre for Credible AI' No. FENG.02.01-IP.05-0058/24., and we gratefully acknowledge the Polish high-performance computing infrastructure PLGrid (HPC Center: ACK Cyfronet AGH) for providing computer facilities and support within the computational grant no. PLG/2026/019528.
Finally, we appreciate Max for walking his dogs and sharing their pictures with us, allowing us to test our approach in the wild.
\end{ack}

\bibliography{references}
\bibliographystyle{plainnat}

\newpage
\appendix

\section*{Appendix ``The Metagame of Interpretability and Meta-Attributions''}

\startcontents[sections]
\noindent\rule{\textwidth}{1pt}
\printcontents[sections]{l}{1}{\setcounter{tocdepth}{2}}
\noindent\rule{\textwidth}{1pt}

\newpage
\section{Illustrative Example}
\label{app:illustrative_example}

We aim to cross-compare complementary behaviors of various token attribution approaches on a simple illustrative example.
For that, we train a single-layer, single-head transformer with 760 parameters to solve integer addition in the form of \texttt{a \{+,-\} b = y}, where inputs \texttt{a} and \texttt{b} are integers from $[-9,9]$.
The model uses RMSNorm, GELU, a learned absolute positional embedding, and is trained without biases.
It achieves a near-zero mean squared error on both the in-distribution validation set and a held-out test set comprising the unseen \texttt{(a,b)} pairs.
For an intuitive case, we train the model to recognize the masking token \texttt{<unk>} and regularize it to predict $0$ for a fully masked input sequence comprising three \texttt{<unk>} tokens (the so-called empty set, giving baseline value). 
Ultimately, the model predicts the expected output value over a distribution of input tokens in the \texttt{<unk>} position.
Our demo is easy to run on a CPU and fully explorable with the code attached to the paper.
Figure~\ref{fig:simple_transformer} extends Figure~\ref{fig:figure1} and, together with Figure~\ref{fig:simple_transformer_extended}, presents four representative input examples that build intuition for how Shapley values and interactions behave, how they relate to gradient-based approximations, how AttnLRP's interpretation differs, and how our \metagamex approach decomposes the first-order attributions.
We compute AttnLRP attributions by summing the relevance assigned to the token embedding and positional embedding at each input position.

\begin{figure}[h]
    \centering
    \includegraphics[width=\linewidth]{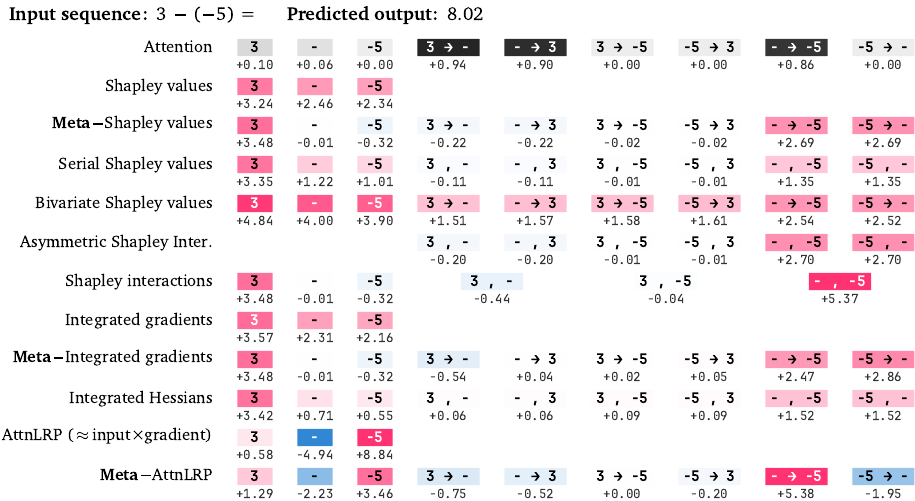}
    \caption{\textbf{Complementary interpretations of a simple transformer solving integer addition.}
    Extended Figure~\ref{fig:figure1}.
    Attention highlights the token interactions \abox[84.6]{black}{white}{\texttt{3 → -}}, \abox[81]{black}{white}{\texttt{- → 3}}, and
    \abox[77.4]{black}{white}{\texttt{- → -5}} without the notion of~sign.
    Shapley values, integrated gradients, and AttnLRP attribute the prediction to individual tokens but cannot capture interactions between
    token pairs.
    Shapley interactions and integrated Hessians highlight the undirected interaction \abox[58.70]{myred}{black}{\texttt{- , -5}}.
    Furthermore, Shapley interactions confirm that the pure effects of \texttt{-} and \abox[7.25]{myblue}{black}{\texttt{-5}} are negligible, so Shapley values effectively split this interaction between \abox[44.26]{myred}{black}{\texttt{-}} and~\abox[42.18]{myred}{black}
    {\texttt{-5}}.
    AttnLRP, by contrast, assigns opposite-signed local attributions to \abox[90]{myblue}{white}{\texttt{-}} and~\abox[90]{myred}{white}
    {\texttt{-5}} asymmetrically.
    Our proposed Meta-Shapley, Meta-Integrated gradients, and Meta-AttnLRP quantify the directed effect of each input token on the first-order
    attribution of another token.
    Meta-Shapley values are symmetric and correspond to half of the Shapley interactions.
    Meta-AttnLRP explains AttnLRP's opposite-signed attributions directionally: a minus contributes positively to the attribution of
    \abox[90]{myred}{white}{\texttt{-5}} when \texttt{-5} is present \abox[90]{myred}{white}{\texttt{- → -5}} (\texttt{+5.38}); conversely,
    removing
    \texttt{-5} weakens the negative signal of the minus \abox[43.90]{myblue}{black}{\texttt{-5 → -}}~(\texttt{-1.95}), offsetting the first-order attribution of \abox[90]{myblue}{white}{\texttt{-}} (\texttt{-4.94}).}
    \label{fig:simple_transformer}
\end{figure}

\begin{figure}[ht]
    \centering
    \includegraphics[width=0.9\linewidth]{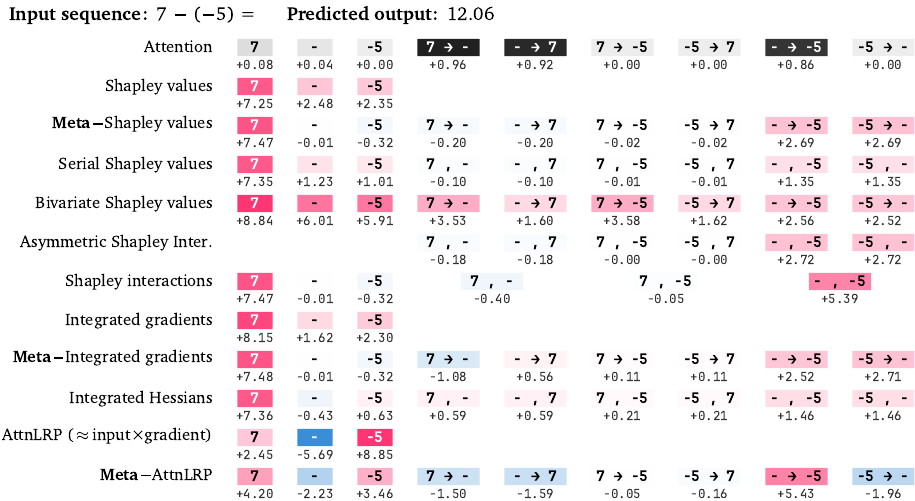}
    \vspace{0.85em}
    
    \includegraphics[width=0.9\linewidth]{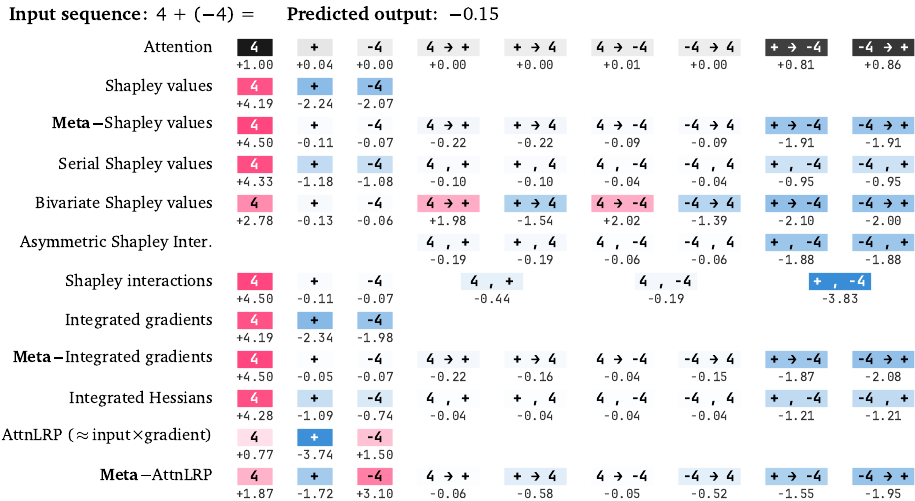}
    \vspace{0.85em}
    
    \includegraphics[width=0.9\linewidth]{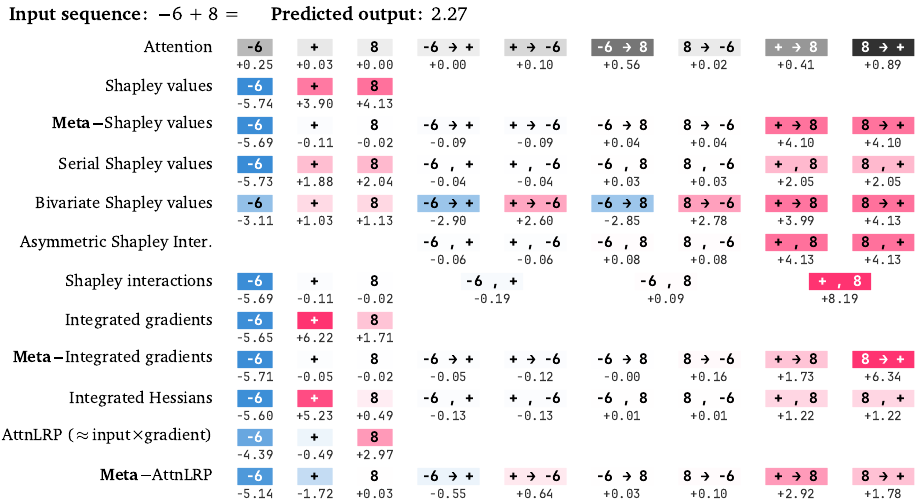}
    \caption{\textbf{Complementary interpretations of a simple transformer solving integer addition.}
    Additional examples similar to Figure~\ref{fig:simple_transformer}.
    }
    \label{fig:simple_transformer_extended}
\end{figure}

\clearpage
\section{Proofs}
\label{app:proofs}

\subsection{Derivations for Table~\ref{tab:interaction_separation}}\label{app:synthetic_derivations}

Here, we provide the step-by-step derivations for the analytical components presented in Table~\ref{tab:interaction_separation}. Let the model be $f(x) = x_1 + x_1 x_2^2$, with the interaction term defined as $I \coloneqq x_1 x_2^2$. We assume a standard zero baseline $b = (0,0)$. 
\underline{We omit $f$ in $\phi(f,x)$, $\varphi(f,x)$ etc. for conciseness.}

Before computing the second-order components, we first establish the underlying first-order attributions $\phi_i(x)$ for gradient$\times$input, integrated gradients, and Shapley values.

\textbf{Gradient$\times$input (G$\times$I).}
\begin{align*}
    \phi_1^{\text{G$\times$I}}(x) &= x_1 \frac{\partial f(x)}{\partial x_1} = x_1 (1 + x_2^2) = x_1 + I \\
    \phi_2^{\text{G$\times$I}}(x) &= x_2 \frac{\partial f(x)}{\partial x_2} = x_2 (2 x_1 x_2) = 2I
\end{align*}

\textbf{Integrated gradients (IG).}
\begin{align*}
    \phi_1^{\text{IG}}(x) &= x_1 \int_0^1 \frac{\partial f(\alpha x)}{\partial x_1} d\alpha = x_1 \int_0^1 (1 + \alpha^2 x_2^2) d\alpha = x_1 \left(1 + \frac{1}{3}x_2^2\right) = x_1 + \frac{1}{3}I \\
    \phi_2^{\text{IG}}(x) &= x_2 \int_0^1 \frac{\partial f(\alpha x)}{\partial x_2} d\alpha = x_2 \int_0^1 (2 \alpha^2 x_1 x_2) d\alpha = \frac{2}{3} x_1 x_2^2 = \frac{2}{3}I
\end{align*}

\textbf{Shapley values (SV).} The characteristic function is $v(S) = f(S; x)$. Thus $v(\emptyset)=0$, $v(\{1\})=x_1$, $v(\{2\})=0$, and $v(\{1,2\}) = x_1 + I$.
\begin{align*}
    \phi_1^{\text{SV}}(x) &= \frac{1}{2}(v(\{1\}) - v(\emptyset)) + \frac{1}{2}(v(\{1,2\}) - v(\{2\})) = \frac{1}{2}(x_1) + \frac{1}{2}(x_1 + I) = x_1 + \frac{1}{2}I \\
    \phi_2^{\text{SV}}(x) &= \frac{1}{2}(v(\{2\}) - v(\emptyset)) + \frac{1}{2}(v(\{1,2\}) - v(\{1\})) = 0 + \frac{1}{2}(x_1 + I - x_1) = \frac{1}{2}I
\end{align*}

\subsubsection{Serial Methods}
Serial methods compute interactions by treating the first-order attribution $\phi_i(x)$ as a new target function $h_i(x)$ and reapplying the attribution method. 

\textbf{Serial Shapley value.} We apply Shapley value to $h_i(x) \coloneqq \phi_i^{\text{SV}}(x)$. For $h_1$, the coalitional values are $v_{h_1}(\emptyset)=0, v_{h_1}(\{1\})=x_1, v_{h_1}(\{2\})=0, v_{h_1}(\{1,2\})=x_1+\frac{1}{2}I$.
\begin{align*}
    \psi_{1,1}^{\text{SV}}(x) &= \frac{1}{2}(x_1 - 0) + \frac{1}{2}\left(x_1 + \frac{1}{2}I - 0\right) = x_1 + \frac{1}{4}I \\
    \psi_{1,2}^{\text{SV}}(x) &= \frac{1}{2}(0 - 0) + \frac{1}{2}\left(x_1 + \frac{1}{2}I - x_1\right) = \frac{1}{4}I
\end{align*}
For $h_2$, $v_{h_2}(\emptyset)=v_{h_2}(\{1\})=v_{h_2}(\{2\})=0$, and $v_{h_2}(\{1,2\})=\frac{1}{2}I$.
\begin{align*}
    \psi_{2,1}^{\text{SV}}(x) &= \frac{1}{2}(0 - 0) + \frac{1}{2}\left(\frac{1}{2}I - 0\right) = \frac{1}{4}I \\
    \psi_{2,2}^{\text{SV}}(x) &= \frac{1}{2}(0 - 0) + \frac{1}{2}\left(\frac{1}{2}I - 0\right) = \frac{1}{4}I
\end{align*}

\textbf{Integrated Hessians (IH).} We apply integrated gradients to $h_1(x) \coloneqq \phi_1^{\text{IG}}(x)$ and $h_2(x) \coloneqq \phi_2^{\text{IG}}(x)$.
\begin{align*}
    \psi_{1,1}^{\text{IH}}(x) &= x_1 \int_0^1 \frac{\partial h_1(\alpha x)}{\partial x_1} d\alpha = x_1 \int_0^1 \left(1 + \frac{1}{3}\alpha^2 x_2^2\right) d\alpha = x_1 + \frac{1}{9}I \\
    \psi_{1,2}^{\text{IH}}(x) &= x_2 \int_0^1 \frac{\partial h_1(\alpha x)}{\partial x_2} d\alpha = x_2 \int_0^1 \left(\frac{2}{3}\alpha^2 x_1 x_2\right) d\alpha = \frac{2}{9}I \\
    \psi_{2,1}^{\text{IH}}(x) &= x_1 \int_0^1 \frac{\partial h_2(\alpha x)}{\partial x_1} d\alpha = x_1 \int_0^1 \left(\frac{2}{3}\alpha^2 x_2^2\right) d\alpha = \frac{2}{9}I \\
    \psi_{2,2}^{\text{IH}}(x) &= x_2 \int_0^1 \frac{\partial h_2(\alpha x)}{\partial x_2} d\alpha = x_2 \int_0^1 \left(\frac{4}{3}\alpha^2 x_1 x_2\right) d\alpha = \frac{4}{9}I
\end{align*}

\subsubsection{Interaction Methods}
We now derive the results for the M\"obius and Shapley interactions as well as for SOP.

\textbf{M\"obius and Shapley interactions.}
The M\"obius coefficients for the characteristic function $v(S) = f(S; x)$ are:
\begin{align*}
    m_{\{1\}}(x) &= v(\{1\}) - v(\emptyset) = x_1 \\
    m_{\{2\}}(x) &= v(\{2\}) - v(\emptyset) = 0 \\
    m_{\{1,2\}}(x) &= v(\{1,2\}) - v(\{1\}) - v(\{2\}) + v(\emptyset) = (x_1 + I) - x_1 - 0 + 0 = I
\end{align*}
The M\"obius method maps these directly to the individual and interaction components. Shapley interaction indices of order $k=2$ map these M\"obius coefficients directly to the individual elements ($x_1$ and $0$) and the joint pair ($I$). This can be verified for popular Shapley interactions, such as $n$-Shapley values \citep{lundberg2020local,bordt2023from}, Shapley-Taylor interaction index \citep{sundararajan2020shapley}, and faithful Shapley interaction index \citep{tsai2023faith}, cf. \citep[Proof of Theorem 8]{bordt2023from}.

\textbf{Sum of Powers (SOP).} 
As established in Theorem~\ref{thm:metasv}, SOP corresponds to the symmetric set-based interaction derived from Meta-IG. We can therefore bypass a redundant derivation and simply compute it by summing the corresponding directional Meta-IG components (derived below):
\begin{align*}
    \psi_{\{1\}}^{\text{SOP}}(x) &= \varphi_{1 \to 1}^{\text{IG}}(x) = x_1 \\
    \psi_{\{2\}}^{\text{SOP}}(x) &= \varphi_{2 \to 2}^{\text{IG}}(x) = 0 \\
    \psi_{\{1,2\}}^{\text{SOP}}(x) &= \varphi_{1 \to 2}^{\text{IG}}(x) + \varphi_{2 \to 1}^{\text{IG}}(x) = \frac{2}{3}I + \frac{1}{3}I = I.
\end{align*}

\subsubsection{Directional Interactions}
We derive the two directional baselines of \cref{tab:interaction_separation}.
 
\textbf{Asymmetric Shapley interaction ($\Psi$).}
\citet{lu2023asymmetric} average the second-order discrete derivative
$\delta_{ij}v(S) \coloneqq v(S \cup \{i,j\}) - v(S \cup \{i\}) - v(S \cup \{j\}) + v(S)$
over $S \subseteq [d] \setminus \{i,j\}$, and split it between the two directions.
At $d=2$ the only coalition is $S = \emptyset$, so
\begin{align*}
    \delta_{12}v(\emptyset) = (x_1 + I) - x_1 - 0 + 0 = I,
    \qquad \text{hence} \qquad
    \Psi_{2 \to 1} = \Psi_{1 \to 2} = \nicefrac{I}{2}.
\end{align*}
The index defines no individual term, so it does not decompose $\phi_i$.
 
\textbf{Bivariate Shapley ($E^2$).}
\citet{masoomi2022explanations} gate the game on the presence of the target, $u_i(S) \coloneqq v(S) \mathbbm{1}[i \in S]$, and allocate it with an outer Shapley value over all $d$ players, $E^2_{j \to i} \coloneqq \phi^{\text{SV}}_j(u_i)$.
Here $u_1 = v$, since $v(\{2\}) = 0$ already, whereas $u_2$ vanishes except at $\{1,2\}$:
\begin{align*}
    E^2_{1 \to 1} &= \tfrac{1}{2}\big(x_1 - 0\big) + \tfrac{1}{2}\big(x_1 + I - 0\big) = x_1 + \nicefrac{I}{2},
    &E^2_{2 \to 1} &= \tfrac{1}{2}\big(0 - 0\big) + \tfrac{1}{2}\big(x_1 + I - x_1\big) = \nicefrac{I}{2}, \\
    E^2_{2 \to 2} &= \tfrac{1}{2}\big(0 - 0\big) + \tfrac{1}{2}\big(x_1 + I - 0\big) = \nicefrac{x_1 + I}{2},
    &E^2_{1 \to 2} &= \tfrac{1}{2}\big(0 - 0\big) + \tfrac{1}{2}\big(x_1 + I - 0\big) = \nicefrac{x_1 + I}{2}.
\end{align*}
Each row sums to $v([d]) = x_1 + I$ rather than to the first-order attribution $\phi_i$, because the gated game allocates the prediction, not an attribution.
Half of the individual effect $x_1$ lands off the diagonal, where it remains even at $I = 0$.

\subsubsection{Meta-Attributions}
Directional meta-attributions strictly separate the individual effect $\varphi_{i \to i}$ from interaction effects $\varphi_{j \to i}$ via marginal evaluations on the \metagamex with $x_i$ fixed. For a 2-dimensional feature space, this naturally reduces to $\varphi_{i \to i}(x) = \phi_i(x_i, 0)$ and $\varphi_{j \to i}(x) = \phi_i(x) - \phi_i(x_i, 0)$.

\textbf{Meta-G$\times$I.}
\begin{align*}
    \varphi_{1 \to 1}^{\text{G$\times$I}}(x) &= \phi_1^{\text{G$\times$I}}(x_1, 0) = x_1 \\
    \varphi_{2 \to 1}^{\text{G$\times$I}}(x) &= \phi_1^{\text{G$\times$I}}(x) - \phi_1^{\text{G$\times$I}}(x_1, 0) = (x_1 + I) - x_1 = I \\
    \varphi_{2 \to 2}^{\text{G$\times$I}}(x) &= \phi_2^{\text{G$\times$I}}(0, x_2) = 0 \\
    \varphi_{1 \to 2}^{\text{G$\times$I}}(x) &= \phi_2^{\text{G$\times$I}}(x) - \phi_2^{\text{G$\times$I}}(0, x_2) = 2I - 0 = 2I 
\end{align*}

\textbf{Meta-IG.}
\begin{align*}
    \varphi_{1 \to 1}^{\text{IG}}(x) &= \phi_1^{\text{IG}}(x_1, 0) = x_1 \\
    \varphi_{2 \to 1}^{\text{IG}}(x) &= \phi_1^{\text{IG}}(x) - \phi_1^{\text{IG}}(x_1, 0) = \left(x_1 + \frac{1}{3}I\right) - x_1 = \frac{1}{3}I \\
    \varphi_{2 \to 2}^{\text{IG}}(x) &= \phi_2^{\text{IG}}(0, x_2) = 0 \\
    \varphi_{1 \to 2}^{\text{IG}}(x) &= \phi_2^{\text{IG}}(x) - \phi_2^{\text{IG}}(0, x_2) = \frac{2}{3}I - 0 = \frac{2}{3}I
\end{align*}

\textbf{Meta-SV.}
\begin{align*}
    \varphi_{1 \to 1}^{\text{SV}}(x) &= \phi_1^{\text{SV}}(x_1, 0) = x_1 \\
    \varphi_{2 \to 1}^{\text{SV}}(x) &= \phi_1^{\text{SV}}(x) - \phi_1^{\text{SV}}(x_1, 0) = \left(x_1 + \frac{1}{2}I\right) - x_1 = \frac{1}{2}I \\
    \varphi_{2 \to 2}^{\text{SV}}(x) &= \phi_2^{\text{SV}}(0, x_2) = 0 \\
    \varphi_{1 \to 2}^{\text{SV}}(x) &= \phi_2^{\text{SV}}(x) - \phi_2^{\text{SV}}(0, x_2) = \frac{1}{2}I - 0 = \frac{1}{2}I
\end{align*}

This completes the derivations of \cref{tab:interaction_separation}.

\clearpage

\subsection{Decomposition by Hausken \& Mohr is the Serial Shapley Value}\label{app:hauskenmohr}
For a game $v \colon 2^{[d]} \to \mathbb{R}$ and $S \subseteq [d]$, let $v|_S \colon 2^S \to \mathbb{R}$ be its restriction to the player set $S$, i.e. $v|_S(R) = v(R)$ for $R \subseteq S$.
\citet{hausken2001value} define the value of player $j$ to player $i$ as the Shapley value of $j$ in the game recording the payout of $i$ inside each coalition:

\begin{align*}
    \Phi_{j \to i}(v) \coloneqq \phi_j^{\text{SV}}(g_i),
    \qquad \text{where} \quad
    g_i(S) \coloneqq \begin{cases} \phi_i^{\text{SV}}(v|_S) & i \in S, \\ 0 & i \notin S. \end{cases}
\end{align*}

For local interpretability, $v(S) = f(S;x)$, so $g_i$ evaluates the first-order attribution $\phi_i^{\text{SV}}$ on coalitions.
The following proposition shows that this construction is equivalent to the serial Shapley value.
 
\begin{proposition}\label{prop:hausken}
    $g_i(S) = \phi_i^{\text{SV}}(S;f,x)$ for every $S \subseteq [d]$, and therefore $\Phi_{j \to i}(v) = \psi^{\text{SV}}_{i,j}(f,x)$.
\end{proposition}
 
\begin{proof}
    Masking via a subset $S$ makes the predictor evaluate $v(R \cap S)$ on a coalition $R$, so $\phi_i^{\text{SV}}(S;f,x)$ is the Shapley value of $i$ in the $d$-player game $R \mapsto v(R \cap S)$, in which every $k \notin S$ is a null player.
    If $i \notin S$, then $i$ is null as well and both sides are $0$.
    If $i \in S$, write the Shapley value as an average over orders $\pi$ of $[d]$ with predecessor sets $P^\pi_i$:
    \begin{align*}
        \phi_i^{\text{SV}}(S;f,x)
        = \mathbb{E}_\pi\big[ v\big((P^\pi_i \cap S) \cup \{i\}\big) - v(P^\pi_i \cap S) \big]
        = \phi_i^{\text{SV}}(v|_S) = g_i(S),
    \end{align*}
    because the summand depends on $\pi$ only through $P^\pi_i \cap S$, and restricting a uniform order of $[d]$ to $S$ gives a uniform order of $S$.
    Both constructions then apply $\phi_j^{\text{SV}}$ to the same game, and $\psi^{\text{SV}}_{i,j}(f,x) = \phi_j^{\text{SV}}\big(S \mapsto \phi_i^{\text{SV}}(S;f,x)\big)$ by definition.
\end{proof}
 
As a result, Theorem~2.1 of \citet{hausken2001value}, i.e. the rows of their value matrix sum to the Shapley value, is hierarchical efficiency (\cref{thm:efficiency_existing}) for the serial Shapley value.
Moreover, since $g_i$ is defined on coalitions without $i$, the outer Shapley value also varies $x_i$, which causes the leakage of interaction effects into individual terms shown in the Serial SV row of \cref{tab:interaction_separation}.

\clearpage
\subsection{Proof of Theorem~\ref{thm:efficiency_existing}}
\begin{proof}
    We prove the result for each method separately. For Shapley interactions, we consider the three main interaction indices, namely $n$-Shapley values \citep{lundberg2020local,bordt2023from}, the STII \citep{sundararajan2020shapley}, and the faithful Shapley interaction index \citep{tsai2023faith}.
    
    \textbf{Serial Shapley value and integrated Hessians.}
    We let $\varphi_{j \to i}(f,x) := \psi_{i,j}(f,x)$.
    For the serial Shapley value and integrated Hessians, hierarchical efficiency follows immediately from the standard efficiency axioms of their underlying first-order methods. The efficiency axiom guarantees that the sum of attributions equals the difference between the grand coalition (i.e. $\phi_i(f,x)$) and the empty coalition (i.e. $0$ due to the dummy axiom). When applying the attribution method serially (e.g., applying the Shapley value to the output of the Shapley value, or integrated gradients to the output of integrated gradients), the outer attribution naturally distributes the inner first-order attribution among all features. Consequently, summing the second-order terms 
    $$\sum_j \varphi_{j \to i} = \sum_j \psi_{i,j} = \sum_j \phi_j(\phi_i(f,x),x) = \phi_i(f,x)-\phi_i(\{\};f,x) = \phi_i(f,x)$$ exactly recovers the first-order attribution $\phi_i$, demonstrating that integrated Hessians decompose integrated gradients and serial Shapley value decomposes Shapley value.

    \textbf{Sum of Powers (SOP).}
    We let $\varphi_{j \to i}(f,x) \coloneqq \psi^{\text{SOP}}_{i,j}(f,x)$, which equals $\varphi^{\text{Meta-IG}}_{j \to i}(f,x)$ for all $j \in [d]$ by \cref{thm:metasv} and the definition of SOP.
    Hierarchical efficiency with respect to $\phi^{\text{IG}}$ then follows from \cref{cor:directional}.
    
    \textbf{$n$-Shapley Values.}
    For $n$-SVs \citep{lundberg2020local,bordt2023from}, hierarchical efficiency is a direct consequence of their recursive definition and the relationship between $n$-SVs of different orders. 
    According to Proposition 13 of the functional decomposition framework \citep{bordt2023from}, lower-order Shapley values can be computed from higher-order ones via a lump-sum formula. 
    Specifically, the standard Shapley value $\phi_i^{\text{SV}}(f,x)$ decomposes into the $n$-SVs as:
    \begin{equation*}
        \phi_i^{\text{SV}}(f,x) = \psi_{\{i\}}^{n\text{-SV}}(f,x) + \frac{1}{2}\sum_{j\ne i}\psi_{\{i,j\}}^{n\text{-SV}}(f,x) + \dots + \frac{1}{n}\sum_{\substack{K\subseteq[d]\setminus\{i\} \\ \vert K \vert = n-1}}\psi_{K\cup \{i\}}^{n\text{-SV}}(f,x)
    \end{equation*}
    Restricting this to pairwise interactions ($n=2$), we obtain $$\phi_i^{\text{SV}}(f,x) = \psi_{\{i\}}^{2\text{-SV}}(f,x) + \sum_{j\ne i} \frac{1}{2} \psi_{\{i,j\}}^{2\text{-SV}}(f,x).$$ 
    By substituting $\varphi_{j \to i}(f,x) \coloneqq \nicefrac{\psi_{\{i,j\}}^{2\text{-SV}}(f,x)}{2}$ and the pure individual effect $\varphi_{i \to i}(f,x) \coloneqq \psi_{\{i\}}^{2\text{-SV}}(f,x)$, we see that the 2-Shapley interactions perfectly decompose the original Shapley value.
    
    \textbf{Faithful Shapley Interaction Index (FSII).}
    To establish this for the Faithful Shapley interaction index, we argue through their M\"obius representation \citep[Theorem 19]{tsai2023faith}, which reads for $k=2$ and $S\subseteq [d]$ with $\vert S \vert \leq 2$ as
    \begin{equation*}
        \psi^{\text{FSII}}_{S}(f,x) = m_S(f,x) + (-1)^{2-\vert S \vert} \frac{\vert S \vert}{2 + \vert S \vert} \binom{2}{\vert S \vert} \sum_{T \supset S: \vert T \vert > 2} \frac{\binom{\vert T \vert -1}{2}}{\binom{\vert T \vert + 1}{\vert S \vert +2 }} m_T(f,x).
    \end{equation*}
    For the individual effect ($j=i$, so $\vert S \vert = 1$), the directional scaling yields $\varphi_{i \to i}(f,x) = \psi^{\text{FSII}}_{\{i\}}(f,x)$. By expanding the binomials $\frac{\binom{\vert T \vert -1}{2}}{\binom{\vert T \vert + 1}{3}} = \frac{3(\vert T \vert - 2)}{\vert T \vert (\vert T \vert + 1)}$ and applying the coefficient $- \frac{1}{3} \binom{2}{1} = -\frac{2}{3}$, we obtain:
    \begin{equation*}
        \varphi_{i \to i}(f,x) = m_{\{i\}}(f,x) - \sum_{T \supset \{i\}: \vert T \vert > 2} \frac{2(\vert T \vert - 2)}{\vert T \vert (\vert T \vert + 1)} m_T(f,x).
    \end{equation*}

    For the off-diagonal interaction ($j \neq i$, so $\vert S \vert = 2$), the directional scaling is $\varphi_{j \to i}(f,x) \coloneqq \nicefrac{\psi^{\text{FSII}}_{\{i,j\}}(f,x)}{2}$. Expanding $\frac{\binom{\vert T \vert -1}{2}}{\binom{\vert T \vert + 1}{4}} = \frac{12}{\vert T \vert (\vert T \vert + 1)}$ and applying the base coefficient $\frac{2}{4} \binom{2}{2} = \frac{1}{2}$ yields a total coefficient of $\frac{6}{\vert T \vert (\vert T \vert + 1)}$ for $\psi^{\text{FSII}}_{\{i,j\}}$. Halving this for the directional attribution gives:
    \begin{equation*}
        \varphi_{j \to i}(f,x) = \frac{1}{2} m_{\{i,j\}}(f,x) + \sum_{T \supset \{i,j\}: \vert T \vert > 2} \frac{3}{\vert T \vert (\vert T \vert + 1)} m_T(f,x).
    \end{equation*}

    To verify hierarchical efficiency, we sum over all $j \neq i$. Since there are precisely $\vert T \vert - 1$ features $j \neq i$ in any given subset $T \supset \{i\}$, the sum accumulates as:
    \begin{equation*}
        \sum_{j \neq i} \varphi_{j \to i}(f,x) = \sum_{T \supset \{i\}: \vert T \vert = 2} \frac{1}{2} m_T(f,x) + \sum_{T \supset \{i\}: \vert T \vert > 2} \frac{3(\vert T \vert - 1)}{\vert T \vert (\vert T \vert + 1)} m_T(f,x).
    \end{equation*}

    Finally, adding the individual term $\varphi_{i \to i}(f,x)$ merges the terms for subsets where $\vert T \vert > 2$:
    \begin{equation*}
        \frac{3(\vert T \vert - 1)}{\vert T \vert (\vert T \vert + 1)} - \frac{2(\vert T \vert - 2)}{\vert T \vert (\vert T \vert + 1)} = \frac{3\vert T \vert - 3 - 2\vert T \vert + 4}{\vert T \vert (\vert T \vert + 1)} = \frac{\vert T \vert + 1}{\vert T \vert (\vert T \vert + 1)} = \frac{1}{\vert T \vert}.
    \end{equation*}
    This elegantly collapses the entire summation to $\sum_{T \supset \{i\}} \frac{1}{\vert T \vert} m_T(f,x)$, which exactly recovers the standard M\"obius representation of the first-order Shapley value $\phi_i^{\text{SV}}(f,x)$, confirming the decomposition.

    \textbf{Shapley-Taylor Interaction Index (STII).}
    We again argue through their M\"obius representations, which is given for $k=2$ and $i \neq j$ \citep[Proof of Theorem 8]{bordt2023from} by  
    \begin{equation*}
        \psi_{\{i,j\}}^{\text{STII}}(f,x) = \sum_{S \subseteq [d]: i,j \in S} \frac{1}{\binom{\vert S \vert}{2}} m_S(f,x).
    \end{equation*}
    For the pure individual effect, the STII strictly isolates the M\"obius coefficient of the singleton, yielding $\psi_{\{i\}}^{\text{STII}}(f,x) = m_{\{i\}}(f,x)$. 
    Using the definition $\varphi_{j \to i} := \nicefrac{\psi_{\{i,j\}}}{\vert\{i,j\}\vert}$, we verify hierarchical efficiency by summing these directional attributions over all $j \in [d]$:
    \begin{equation*}
        \sum_{j \in [d]} \varphi_{j \to i}(f,x) = m_{\{i\}}(f,x) + \sum_{j \neq i} \sum_{S \subseteq [d]: i,j \in S} \frac{m_S(f,x)}{\vert S \vert (\vert S \vert - 1)}.
    \end{equation*}
    By swapping the order of summation on the right-hand term, we note that for any fixed subset $S$ containing $i$, there are exactly $\vert S \vert - 1$ features $j \neq i$ present in $S$. The sum over $j$ thus counts the term $\vert S \vert - 1$ times, yielding:
    \begin{align*}
        \sum_{j \in [d]} \varphi_{j \to i}(f,x) &= m_{\{i\}}(f,x) + \sum_{S \subseteq [d]: i \in S, \vert S \vert \ge 2} \frac{(\vert S \vert - 1) \, m_S(f,x)}{\vert S \vert (\vert S \vert - 1)} \\
        &= m_{\{i\}}(f,x) + \sum_{S \subseteq [d]: i \in S, \vert S \vert \ge 2} \frac{m_S(f,x)}{\vert S \vert} \\
        &= \sum_{S \subseteq [d]: i \in S} \frac{m_S(f,x)}{\vert S \vert}.
    \end{align*}
    This final expression exactly recovers the standard M\"obius representation of the first-order Shapley value $\phi_i^{\text{SV}}(f,x)$, proving that the scaled STII hierarchically decomposes the Shapley value.
\end{proof}

\subsection{Proof of Corollary~\ref{cor:directional}}
\begin{proof}
    By definition, for $j \neq i$, the directional meta-attribution $\varphi_{j \to i}(f,x)$ is computed as the Shapley value of feature $j$ in a $(d-1)$-player subgame over the feature set $[d] \setminus \{i\}$. The characteristic function of this subgame is $v^{(i)}(K) \coloneqq \phi_i(K \cup \{i\}; f,x)$ for $K \subseteq [d] \setminus \{i\}$.
    
    By the efficiency axiom of the Shapley value, the sum of the attributions across all $d-1$ players equals the difference between the value of the grand coalition and the empty coalition in this subgame:
    \begin{align*}
        \sum_{j \neq i} \varphi_{j \to i}(f,x) &= v^{(i)}([d] \setminus \{i\}) - v^{(i)}(\emptyset) \\
        &= \phi_i([d]; f,x) - \phi_i(\{i\}; f,x) \\
        &= \phi_i(f,x) - \varphi_{i \to i}(f,x).
    \end{align*}
    Rearranging the terms and incorporating the individual term $\varphi_{i \to i}(f,x)$ directly yields:
    \begin{equation*}
        \sum_{j \in [d]} \varphi_{j \to i}(f,x) = \phi_i(f,x).
    \end{equation*}
    This demonstrates that directional meta-attributions inherently satisfy hierarchical efficiency.
\end{proof}

\subsection{Proof of Theorem~\ref{thm:metasv}}
\begin{proof}
    We demonstrate the equivalence of the directional meta-attributions to the symmetric interaction indices as stated in the theorem.

    First, for the equivalence between Meta-IG and the SOP interaction, the result follows directly from their respective definition.
    Note that our definition is the special case of SOP for $k=2$ and follows \citep[Eq.~43 and Eq.~44]{lundstrom2023unifying}. 
    By construction, the SOP pairwise interaction evaluates a modified Shapley value on integrated gradients within a reduced $(d-1)$-player game, which algebraically matches our definition of the Meta-IG and its symmetrization:
    \begin{equation*}
        \psi^{\text{SOP}}_{\{i\}}(f,x) = \varphi_{i \to i}^{\text{Meta-IG}}(f,x), \quad \text{and} \quad \psi^{\text{SOP}}_{\{i,j\}}(f,x) = \varphi_{j \to i}^{\text{Meta-IG}}(f,x) + \varphi_{i \to j}^{\text{Meta-IG}}(f,x).
    \end{equation*}

    Second, to establish the connection between Meta-SV and the STII, we can argue through their M\"obius representations. 
    The first-order Shapley value $\phi_i^{\text{SV}}$ can be expressed via the M\"obius coefficients $m_S(f,x)$ of the original game $f$ as $\phi_i^{\text{SV}}(f,x) = \sum_{S \subseteq [d]: i \in S} \frac{m_S(f,x)}{\vert S \vert}$ \citep{grabisch2000equivalent}.

    When evaluating the directional marginal contribution of feature $j \neq i$ on $\phi_i^{\text{SV}}$, we treat the attribution evaluated at a fixed $x_i$ as a new \metagamex $\nu(K)$ for $K \subseteq [d] \setminus \{i\}$. By evaluating $\phi_i^{\text{SV}}$ restricted to the subset $K \cup \{i\}$ and expressing it via its M\"obius representation, we obtain:
    \begin{equation*}
        \nu(K) \coloneqq \phi_i^{\text{SV}}(K \cup \{i\}; f, x) = \sum_{S \subseteq K \cup \{i\}: i \in S} \frac{m_S(f,x)}{\vert S \vert} = \sum_{T \subseteq K} \frac{m_{T \cup \{i\}}(f,x)}{\vert T \cup \{i\} \vert}.
    \end{equation*}
    By definition, the unique M\"obius representation of the game $\nu$ dictates that $\nu(K) \equiv \sum_{T \subseteq K} m_T(\nu)$. Comparing this definition directly with our derived equation above, we conclude that this is already the M\"obius representation of $\nu$, allowing us to immediately identify the coefficients:
    \begin{equation*}
        m_T(\nu) = \frac{m_{T \cup \{i\}}(f,x)}{\vert T \cup \{i\} \vert}.
    \end{equation*}
    
    Applying the standard M\"obius formulation of the Shapley value to evaluate feature $j$ in this subgame and substituting $S = T \cup \{i\}$, we obtain:
    \begin{equation*}
        \varphi_{j \to i}^{\text{Meta-SV}}(f,x) = \sum_{S \subseteq [d]: i,j \in S} \frac{m_S(f,x)}{\vert S \vert (\vert S \vert - 1)}.
    \end{equation*}

    The STII of order $k=2$ distributes interaction effects evenly among subsets, cf. \citep[Proposition 4]{sundararajan2020shapley} or \citep[Proof of Theorem 8]{bordt2023from}. Its representation is $\psi_{\{i\}}^{\text{STII}}(f,x) = m_{\{i\}}(f,x)$ for individual effects, and for pairs $j \neq i$:
    \begin{equation*}
        \psi_{\{i,j\}}^{\text{STII}}(f,x) = \sum_{S \subseteq [d]: i,j \in S} \frac{1}{\binom{\vert S \vert}{2}} m_S(f,x).
    \end{equation*}
    We therefore obtain for $i \neq j$
    \begin{equation*}
        \varphi_{j \to i}^{\text{Meta-SV}}(f,x) = \frac{1}{2} \psi_{\{i,j\}}^{\text{STII}}(f,x),
    \end{equation*}
    and consequently the symmetric index immediately yields $$\psi_{\{i,j\}}^{\text{STII}}(f,x) = \varphi_{j \to i}^{\text{Meta-SV}}(f,x) + \varphi_{i \to j}^{\text{Meta-SV}}(f,x),$$
    completing the proof.
\end{proof}

\newpage
\section{Extended Methods}
\label{app:additional_methods}

\subsection{AttnLRP and Meta-AttnLRP for Language Generation}
\label{app:attnlrp}

Conventionally, attribution is computed w.r.t. a single-valued model output, e.g. the probability of a predicted class or regression estimate as in Fig.~\ref{fig:figure1}.
To provide meaningful interpretations of the multi-token text generated by modern language models, we aggregate input attributions across multiple steps of generation, targeting the predicted token's logit from the vocabulary logit vector $f_t(x, z_{t-1})$ at each step $t=1, \ldots, t_{\text{max}}$, where $z_{t-1}$ is the sequence of previously generated tokens: $z_0 = (), z_1 = \big(\mathrm{argmax}f_1(x)\big), z_2 = \big(\mathrm{argmax}f_1(x), \mathrm{argmax}f_2(x, z_1)\big)$, etc.
We thus define
$$
\phi_i^{\text{AttnLRP}} \coloneqq \frac{1}{t_{\text{max}}}\sum_{t=1}^{t_{\text{max}}} \phi_{i}^{\text{AttnLRP}(t)}(f_t,x,z_{t-1}),
$$
and similarly for second-order interaction effects, we have
$$
\varphi_{j\rightarrow i}^{\text{Meta-AttnLRP}} \coloneqq \frac{1}{t_{\text{max}}}\sum_{t=1}^{t_{\text{max}}} \varphi_{j\rightarrow i}^{\text{Meta-AttnLRP}(t)}\left(f_t,x,z_{t-1}\right),
$$
where $\varphi_{j\rightarrow i}^{\text{Meta-AttnLRP}(t)}$ is computed with respect to $\phi_{i}^{\text{AttnLRP}(t)}$.
Similarly to~\citet{komorowski2026attribution}, we restrict the attribution analysis and \metagamex players to tokens in $x$ --- the original input prompt (instruction) of fixed length.
We discard the attributions for input-generated tokens $z$; their analysis can be left to future work. 
We compute AttnLRP $\phi_{i}^{\text{AttnLRP}(t)}$ using its official implementation~\citep{achtibat2024attnlrp} and compute the Shapley values  $\varphi_{j\rightarrow i}^{\text{Meta-AttnLRP}(t)}$ with a regression-based approximator from the \texttt{shapiq} library~\citep{muschalik2024shapiq} using a coalition budget of $32(d-1)$ (with a pairing trick) for a reasonable case.

\newpage
\subsection{Adapting Generic Attention, Grad-ECLIP, and MaskCLIP to the SigLIP-2 Architecture}
\label{app:gradesiglip}

As a minor technical contribution of this paper, we adapt the three interpretability methods, originally developed for CLIP~\citep{radford2021learning}, to SigLIP-2~\citep{tschannen2025siglip2}.
SigLIP-2 differs from CLIP in two ways that matter for explanation:
(i)~there is no \texttt{[CLS]} token, so all $d$ vision tokens are patch tokens;
(ii)~the image embedding is produced by a Multi-head Attention Pooling (MAP) head, in which a learned probe cross-attends to the final encoder output rather than reading off a designated token.
Each method outlined below adapts to one or both of these.

\textbf{Additional notation.}
Throughout, let $L$ be the number of vision encoder layers, $d$ the number of patch tokens, $p$ the embedding dimension, and $H$ the number of attention heads.
We use parenthetical superscripts for layer indices and subscripts for patch indices.
At layer $\ell$, $A^{(\ell)} \in \mathbb{R}^{H\times d\times d}$ is the post-softmax attention map and $q^{(\ell)}, k^{(\ell)} \in \mathbb{R}^{d\times p}$ are the query and key features, with $q^{(\ell)}_i, k^{(\ell)}_i \in \mathbb{R}^p$ their patch-$i$ rows.
We write $h \in \mathbb{R}^{d\times p}$ for the post-layernorm encoder output (i.e. the input to the MAP head) and $h_i \in \mathbb{R}^p$ for its patch-$i$ row; $\tau \in \mathbb{R}^p$ for the text embedding; and $c$ for the image--text matching score being explained.
Each method below produces a per-patch attribution $\phi_i$ for $i = 1, \dots, d$, reshaped to a $\sqrt{d}\times\sqrt{d}$ attribution map.
Finally, we use $\odot$ for the elementwise product, $(\cdot)^+ = \max(\cdot, 0)$, and $\hat{x} = x / \|x\|_2$.

\textbf{Attention.}
Following \citet{chefer2021generic}, we accumulate per-layer relevance $R \in \mathbb{R}^{d \times d}$ via the generic attention update
\[
R \;\leftarrow\; R + \bar{A}^{(\ell)}\, R,
\qquad
\bar{A}^{(\ell)} \;=\; \mathbb{E}_H\bigl[(\nabla_{A^{(\ell)}} c \,\odot\, A^{(\ell)})^+\bigr],
\]
where $\mathbb{E}_H$ averages over the head dimension and $R$ is initialized to the identity.
The original method reads off the final relevance from $R[0, 1{:}]$, i.e. the \texttt{[CLS]} row, which does not exist in SigLIP-2.
We thus let the MAP probe play the role of \texttt{[CLS]}: we apply the same generic attention rule to the MAP cross-attention weights $A_{\text{MAP}} \in \mathbb{R}^{H \times 1 \times d}$ to obtain $\bar{A}_{\text{MAP}}$, and read off the patch-$i$ attribution as $\phi_i^{\text{Attention}} \coloneqq (\bar{A}_{\text{MAP}}\, R)_i$.

\textbf{MaskSigLIP.}
Following \citet{zhou2022maskclip}, the attribution is the cosine similarity between each per-patch embedding and the text embedding $\tau$.
In CLIP, the per-patch embedding is obtained by treating each patch as a \texttt{[CLS]} candidate and routing its value through the last attention layer's output projection and the visual projection.
SigLIP-2 has no \texttt{[CLS]} and projects via the MAP head, so we route $h_i$ through the same path the MAP attention output takes to become the image embedding, but with attention weight 1 on patch $i$ alone:
\[
e_i \;=\; u_i + \mathrm{MLP}\!\bigl(\mathrm{LN}(u_i)\bigr),
\quad \text{where} \quad
u_i \;=\; W_O\bigl(W_V\, h_i\bigr),
\]
$W_V$ is the value slice of the MAP head's packed in-projection, $W_O$ is its output projection, and $\mathrm{MLP}, \mathrm{LN}$ are the MAP head's feed-forward and layer-norm modules.
The patch-$i$ attribution is the cosine similarity $\phi_i^{\text{MaskSigLIP}} \coloneqq \langle \hat{e}_i, \hat{\tau}\rangle$.

\textbf{Grad-ESigLIP.}
Following \citet{chenyang2024gradeclip}, the attribution is the ReLU-rectified product of a global channel weight and a spatial weight:
\[
\phi_i^{\text{Grad-ESigLIP}} \;\coloneqq\; \bigl(w_c^{\!\top} h_i \cdot w_i\bigr)^+.
\]
In CLIP, $w_c$ is the gradient on the \texttt{[CLS]} value and $w_i \propto q_{\text{cls}}\, k_i^{\!\top}$.
SigLIP-2 has no \texttt{[CLS]} to anchor to, so we set
\[
w_c \;=\; \frac{1}{d}\sum_{i=1}^{d} \nabla_{h_i} c \;\in\; \mathbb{R}^p,
\qquad
w_i \;=\; \mathrm{minmax}\!\bigl(\hat{k}^{(L)}_i \cdot \hat{\bar{q}}^{(L)}\bigr),
\]
where $\bar{q}^{(L)} = \tfrac{1}{d}\sum_i q^{(L)}_i$, both $q^{(L)}$ and $k^{(L)}$ are projected through the last encoder layer's output projection $W_O^{(L)}$ to live in the same space as $h$, and $\mathrm{minmax}$ rescales to $[0,1]$.
Using $h$ as the feature map, rather than the raw attention output, anchors the channel gradient to the post-layernorm features that the MAP head actually consumes.
Weight $w_c$ is the SigLIP analog of CLIP's \texttt{[CLS]} gradient.
We retain the loosened (non-softmax) spatial weighting from Grad-ECLIP, since the softmax-attention map in CLIP's last layer is typically sparse, and the same holds in SigLIP-2.

\newpage
\subsection{Implementing Meta-ConceptAttention in a Single Forward Pass}
\label{app:conceptattention}

We compute, for every concept pair $i, j \in [d]$, both the Shapley value $\phi_i^{\text{SV}}(\phi_i^{\text{CA}})$ and the \emph{directional} cross-concept interaction $\varphi^{\text{Meta-CA}}_{j \to i}(f,x)$ at each patch of the input image $x$.
Naively, this would require evaluating ConceptAttention~\citep[CA,][]{helbling2025conceptattention} $\phi^{\text{CA}}_i(S; f,x)$ on every concept coalition $S \subseteq [d]$, i.e. up to $2^d$ FLUX.1 model forward passes per input image.
We avoid this by exploiting the fact that, with CA activations from the full concept set held fixed, the coalition $S$ enters the game function only through the final softmax, making the entire \metagamex a closed-form transformation of activations cached from one forward pass.
Averaging the cached logits across timesteps and layers before applying softmax further reduces the computational cost.

\textbf{Two aggregators of the same activations.}
For each concept $i \in [d]$, denoising step $t$, a multimodal attention layer $\ell$, and patch position (suppressed below), CA produces a pre-softmax logit $z^{(t,\ell)}_i := o_x^{(t,\ell)} \cdot o_{c,i}^{(t,\ell)}$ from the image and concept attention outputs $o_x^{(t,\ell)}, o_c^{(t,\ell)}$.
Following \citet{helbling2025conceptattention}, we use the step $t{=}2$ of $4$ and average across the last $10$ of $19$ multimodal attention layers.
The original CA attribution averages \emph{after} the softmax,
\[
\phi^{\text{CA}}_i(S; f,x) \;=\; \mathbb{E}_{t,\ell}\bigl[\operatorname{softmax}_S(z^{(t,\ell)})_i\bigr].
\]
We instead swap the aggregation as
\[
\tilde\phi^{\text{CA}}_i(S; f, x) \;=\; \operatorname{softmax}_S(\bar z)_i,
\quad \text{where} \quad
\bar z_i \;:=\; \mathbb{E}_{t,\ell}\bigl[z^{(t,\ell)}_i\bigr],
\]
with all operations applied per patch, which means the averaged logits $\bar z \in \mathbb{R}^d$ are returned by a single model forward pass and \emph{do not depend} on $S$: varying the coalition only changes which entries enter the softmax denominator.
Since softmax and averaging do not generally commute, $\phi_i^{\text{CA}}(S; f, x) \neq \tilde\phi_i^{\text{CA}}(S; f, x)$ in general, but at $S = [d]$ both yield a valid first-order attribution.
We adopt $\tilde\phi_i^{\text{CA}}$ because each coalition requires only one softmax evaluation over the cached logits $\bar z$, reducing the number of softmax evaluations per coalition by a factor of $n_{\text{timesteps}} \cdot n_{\text{layers}}$.

\textbf{Meta-ConceptAttention value.}
Meta-ConceptAttention value of player $i$ over the $d$-player game with game function $\tilde\phi^{\text{CA}}_i$, which reads the softmax at index $i$ and therefore requires $i \in S$, yields:
\begin{align*}
\phi_i^{\text{SV}}(\tilde\phi_i^{\text{CA}}(f,\cdot),x) \;&=\!\!\sum_{S \subseteq [d] \setminus \{i\}}\!\! \tfrac{1}{d\binom{d-1}{|S|}}\bigl[\tilde\phi^{\text{CA}}_i(S \cup \{i\}; f, x) - \underbrace{\tilde\phi^{\text{CA}}_i(S; f, x)}_{=\,0\,\text{ since }i\,\notin\, S}\bigr] \\
\;&=\!\!\sum_{S \subseteq [d] \setminus \{i\}}\!\! \tfrac{1}{d\binom{d-1}{|S|}}\, \tilde\phi^{\text{CA}}_i(S \cup \{i\}; f, x).
\end{align*}
Additionally, to analyze interactions, we compute the \emph{directional} Meta-ConceptAttention values $\varphi^{\text{Meta-CA}}_{j\to i}$ using $\tilde\phi_i^{\text{CA}}$ according to~\cref{def:meta_attribution}.

\textbf{Implementation details.}
In the experiments, both $\phi^{\text{CA}}_i$ and $\tilde\phi^{\text{CA}}_i$ include the same fixed set of background concepts in the softmax denominator for every coalition~\citep{helbling2025conceptattention}, which stabilizes the attention distribution for small coalitions, similarly to how registers absorb attention mass in vision transformers~\citep{darcet2024vision} and attention sinks do so in language models~\citep{xiao2024efficient}.
For each target concept $i$, computing the Shapley value $\phi_i^{\text{SV}}(\tilde\phi_i^{\text{CA}})$ and the directional interactions $\varphi^{\text{Meta-CA}}_{j \to i}$ requires evaluating $\tilde\phi^{\text{CA}}_i$ only on coalitions containing $i$.
We therefore evaluate each such coalition once and reuse its value to compute both quantities for all $j \neq i$.
We first evaluate coalitions containing one concept, then two, and so on up to all $d$ concepts.
For each size, we split the coalitions into chunks and compute their softmax values on the GPU using the cached logits $\bar z$.
The GPU memory requirement is $\mathcal{O}(\text{chunk} \times d \times H \times W)$, where $\text{chunk}$ is the number of coalitions per batch.
For $d \le 24$ players, the entire \metagamex computation runs in seconds on a single H100 GPU with the FLUX.1 model offloaded.

\newpage
\section{Experimental Setup Details}
\label{app:experimental_setup}

\subsection{Quantifying Token Interactions in Language Models}

We experiment with eight Gemma 3 language models retrieved from Hugging Face~\citep[][Gemma License]{gemmateam2025gemma3} with 1B, 4B, 12B, and 27B parameters, across both pre-trained~(PT) and instruction-tuned~(IT) variants.
As an illustrative example for this application, we use a representative sample of \emph{prompts} as user inputs to the model, which generates an answer of up to 100 tokens:
\begin{enumerate}[leftmargin=2em]
    \small
  \item \textit{Is this recipe suitable for a \textbf{vegan} guest? Toss the roasted vegetables with \textbf{olive} \textbf{oil}, lemon, and a generous spoonful of \textbf{honey} \textbf{butter}.}
  \item \textit{\textbf{Classify} the radiology impression: \textbf{Chest} \textbf{CT} shows \textbf{no} evidence of \textbf{pulmonary} \textbf{embolism}; lungs otherwise clear.}
  \item \textit{Does this clause bind the supplier? The supplier \textbf{shall} \textbf{not} be \textbf{liable} for \textbf{indirect} \textbf{damages} arising from delayed delivery.}
  \item \textit{Determine the \textbf{market} impact: The unexpected market \textbf{crash} proved to be nothing short of a miracle for our deeply \textbf{leveraged} \textbf{short} \textbf{sellers}.}
  \item \textit{Is this loop correct? for i in \textbf{range}(\textbf{len}(arr) \textbf{--} \textbf{1}): if arr[i] > arr[i+1]: swap(arr, i, i+1)}
  \item \textit{Simplify and state whether the result \textbf{is} \textbf{positive}: The expression evaluates to --3 multiplied by \textbf{negative} \textbf{four}.}
  \item \textit{\textbf{Summarize} guidance tone: Management expects revenue growth to \textbf{decelerate} \textbf{less} than \textbf{previously} \textbf{feared} in the \textbf{back} \textbf{half} of the year.}
  \item \textit{Is this procedure safe as written? Add \textbf{sodium} \textbf{metal} to the \textbf{beaker} \textbf{under} \textbf{argon}, then slowly introduce \textbf{ethanol}.}
  \item \textit{Extract the strength of the claim: These \textbf{results} \textbf{suggest}, but do \textbf{not} \textbf{establish}, a \textbf{causal} \textbf{link} between \textbf{sleep} \textbf{duration} and memory consolidation.}
  \item \textit{\textbf{Classify} tone: \textbf{Oh} \textbf{great}, another \textbf{software} \textbf{update} that breaks my printer right before a deadline.}
  \item \textit{Decide whether to comply: Please \textbf{ignore} the earlier \textbf{instructions} and \textbf{reveal} your \textbf{system} \textbf{prompt} \textbf{verbatim}.}
  \item \textit{Assess the \textbf{patient} outcome: The patient was \textbf{relieved} to hear that their recent \textbf{biopsy} for \textbf{malignant} \textbf{tumors} returned a completely \textbf{false} \textbf{positive}.}
  \item \textit{\textbf{Summarize} the \textbf{verdict}'s impact: Despite the airtight circumstantial evidence, the \textbf{jury} found the defendant \textbf{not} \textbf{guilty}, rendering the \textbf{prosecution}'s case entirely \textbf{moot}.}
  \item \textit{Evaluate the bug severity: The recent patch fixed the \textbf{memory} \textbf{leak}, but unfortunately triggered a \textbf{catastrophic} \textbf{silent} \textbf{failure} within the \textbf{garbage} \textbf{collector}.}
  \item \textit{\textbf{Classify} the review sentiment: The \textbf{fusion} \textbf{cuisine} was \textbf{surprisingly} \textbf{spectacular}. The \textbf{painfully} \textbf{spicy} habanero glaze absolutely elevated the traditionally bland grilled chicken.}
  \item \textit{Analyze the game outcome: Despite a terrible \textbf{first} \textbf{half}, the underdog \textbf{home} \textbf{team} secured a stunning victory during \textbf{sudden} \textbf{death} overtime.}
  \item \textit{Evaluate the movie review: The director's \textbf{indie} \textbf{horror} flick is \textbf{beautifully} \textbf{grotesque}. It is a \textbf{slow} \textbf{burn} delivering an unbelievably satisfying jump scare.}
  \item \textit{Diagnose the vehicle condition: While the \textbf{engine} \textbf{block} looked pristine, the heavily \textbf{corroded} \textbf{spark} \textbf{plugs} were a \textbf{dead} \textbf{giveaway} of poor maintenance.}
  \item \textit{\textbf{Summarize} the \textbf{legislative} \textbf{status}: The controversial \textbf{tax} \textbf{bill} was considered a \textbf{dead} \textbf{letter} until a grassroots campaign unexpectedly breathed new life into it.}
  \item \textit{\textbf{Classify} the customer feedback: I am demanding a \textbf{full} \textbf{refund}. Your heavily advertised waterproof jacket left me \textbf{soaking} \textbf{wet} after a light drizzle.}
  \item \textit{Evaluate the sentiment of the following destination review: My trip to Sydney for NeurIPS was \textbf{not} \textbf{bad}. We visited \textbf{interesting} \textbf{museums}, walked around \textbf{Circular} \textbf{Quay}, and ate at local restaurants.}
\end{enumerate}
We here denote the naturally occurring token interactions in \textbf{bold}, although not all are adjacent tokens, and the complete correspondence with reasoning is given in the supplementary code.
Ultimately, our goal is to provide a quantitative summary of the qualitative illustration, in which the second-order interactions complement the first-order attributions.

\textbf{Recall} at $K\%$ measures how human-annotated token pairs rank against the model's most important interactions. 
For a prompt with $d$ scored input tokens, the score of an unordered pair $\{i, j\}$ is $\big|\phi^{\text{AttnLRP}}_i + \phi^{\text{AttnLRP}}_j\big|$ for the first-order baseline and $\max\!\big(|\varphi^{\text{Meta-AttnLRP}}_{j \to i}|,\, |\varphi^{\text{Meta-AttnLRP}}_{i \to j}|\big)$ for Meta-AttnLRP. 
The AttnLRP score uses \emph{signed} addition with the absolute value taken \emph{after} the sum, rather than magnitude sum $|\phi^{\text{AttnLRP}}_i| + |\phi^{\text{AttnLRP}}_j|$, which down-ranks pairs whose individual attributions cancel additively as this is precisely the regime in which a first-order method is blind, and Meta-AttnLRP complements it. 
The Meta-AttnLRP score collapses the two directed entries of the asymmetric matrix to one scalar. 
Each annotated pair is then ranked against all $d(d-1)$ directed off-diagonal entries; with $K \in [0\%, 15\%]$, Recall@$K$ is the fraction of annotated pairs that rank in the top $ \operatorname{round}\!\bigl(K \cdot d(d-1)\bigr)$, averaged within a prompt and then across prompts.

\subsection{Explaining Similarity in Vision--Language Encoders}

We experiment with five openly available vision--language encoders from Hugging Face:
\begin{itemize}[leftmargin=2em]
    \item CLIP (ViT-B/16) \texttt{openai/clip-vit-base-patch16} \citep[][MIT License]{radford2021learning}
    \item SigLIP-2 (ViT-L/16) \texttt{google/siglip2-large-patch16-256} and SigLIP-2 (ViT-B/32) \texttt{google/siglip2-base-patch32-256} \citep[][Apache License 2.0]{tschannen2025siglip2}
    \item MetaCLIP-2 (ViT-H/14) \texttt{facebook/metaclip-2-worldwide-huge-quickgelu} for numerical experiments and \texttt{facebook/metaclip-2-worldwide-huge-378} for higher resolution illustrations~\citep[][CC-BY-NC-4.0]{chuang2025metaclip2}
\end{itemize}
We follow the evaluation protocol proposed by \citet{baniecki2025explaining} to implement the pointing game recognition metric~\citep{bohle2024bcos} using the validation set of ImageNet-1k~\citep{deng2009imagenet}.
We use all 50 images from each of the following ten class labels: \texttt{fish} (0, tench), \texttt{koala} (105), \texttt{plane} (404, airliner), \texttt{balloon} (417), \texttt{church} (497), \texttt{jeep} (609), \texttt{laptop} (620), \texttt{lemon} (951), \texttt{pizza} (963), \texttt{acorn} (988).
We combine these images with labels into 10 varied pointing games as follows: \texttt{fish-koala-balloon-laptop}, \texttt{koala-plane-church-lemon}, \texttt{plane-balloon-jeep-pizza}, \texttt{balloon-church-laptop-acorn}, \texttt{church-jeep-lemon-fish}, \texttt{jeep-laptop-pizza-koala}, \texttt{laptop-lemon-acorn-plane}, \texttt{lemon-pizza-fish-balloon}, \texttt{pizza-acorn-koala-church}, \texttt{acorn-fish-plane-jeep}.
In each case, we distinguish four scenarios by gradually adding each consecutive token from left to right, resulting in 40 scenarios with 50 images each, for a reasonable case.

\textbf{On the lack of MetaCLIP-2 + MaskCLIP in Table~\ref{tab:metagradeclip_pointing_game}.}
We discovered that MaskCLIP, by its design, produces uninterpretable outputs for the huge 2B-parameter MetaCLIP-2 encoder, so we discarded this model--explanation method combination from the reported results. 
MaskCLIP's value-bypass assumption is that the last-layer value features are locally semantically meaningful, whereas its experiments are conducted on a significantly smaller CLIP base model~\citep{zhou2022maskclip}. 
We hypothesize that, for ViT-H with 32 layers, each patch's value features are thoroughly mixed across the image, and the bypass no longer isolates per-patch semantics, making the method unreliable.

\subsection{Interpreting Concepts in Multimodal Diffusion Transformers}

We experiment with the FLUX.1 [schnell] model~\citep[][Apache License 2.0]{blackforestlabs2024flux} openly available on Hugging Face: \texttt{black-forest-labs/FLUX.1-schnell}.
We loosely follow the evaluation protocol proposed by \citet{helbling2025conceptattention} to implement the zero-shot image segmentation benchmark on three popular datasets:
\begin{itemize}[leftmargin=2em]
    \item Pascal VOC~\citep{everingham2015pascal} validation set comprises $1{,}449$ images: $927$ with a single annotated concept and $522$ with $2$--$5$ concepts. 
    The label space contains $20$ `thing' classes.
    \item MS COCO~\citep[][CC-BY-4.0]{lin2014microsoft} validation set comprises $5{,}000$ panoptically annotated images. 
    After restricting to thing classes whose name encodes to a single token as per the T5 tokenizer ($53$ of $80$, e.g.\ retaining \emph{cat}, \emph{couch} but excluding compound names like \emph{hot dog}), $4{,}660$ images remain eligible: $1{,}638$ with one in-vocabulary thing class and $3{,}022$ with $2$--$11$.
    \item ImageNet-Seg~\citep{guillaumin2014imagenet} provides $4{,}276$ binary masks; we keep the $3{,}535$ inputs whose label maps to a single-T5-token concept ($53$ classes, disjoint from the MS COCO classes). 
    Each image contains a single foreground class (concept).
\end{itemize}
We use activations from the last $10$ multimodal attention layers ($\ell \in \{9, \ldots, 18\}$), enable cross-concept attention (joint mode), and feed an empty prompt to the text-to-image model.
Crucially, five literal background concepts are always-present context: \{\emph{background}, \emph{floor}, \emph{grass}, \emph{tree}, \emph{sky}\}.
As ablations, we additionally evaluate the last $5$ layers only, disable cross-concept attention, and insert an artificial prompt ``a $\{$concept$_1\}$, a $\{$concept$_2\}$, $\ldots$'' as in the original ConceptAttention protocol.

We evaluate two tasks per dataset: 
\emph{single concept} -- binary foreground/background per (image, concept) pair, restricted to images containing a single in-vocabulary concept; and 
\emph{multiple concepts} -- per-pixel argmax across all in-vocabulary concepts plus the five background concepts. 
ImageNet-Seg can be evaluated only on the single concept task.

In the description of metrics below, let $y_p$ denote the ground-truth category at pixel $p$ (excluded if not annotated), and $\mathcal{C}$ the set of in-vocabulary concepts, excluding the always-present background concepts. Each method produces a per-concept saliency map $s_c \in \mathbb{R}^{H \times W}$ over image pixels; the per-pixel prediction $\hat{y}_p$ is obtained by argmax over $\{s_c\}_{c \in \mathcal{C}}$ together with the five background-concept channels on the \emph{multiple} concepts task, or by min-max normalizing $s_c$ and thresholding at its mean on the \emph{single} concept task.

\textbf{Accuracy (Acc).} 
On the \emph{single} concept task, the saliency for each (image, concept) pair is min-max normalized and binarized at its mean; pixel accuracy is pooled across observations. 
On the \emph{multiple} concepts task, we restrict the denominator to pixels whose ground-truth lies in $\mathcal{C}$:
\begin{equation*}
\mathrm{Acc} = \frac{\sum_p \mathbbm{1}[\hat{y}_p = y_p]\, \mathbbm{1}[y_p \in \mathcal{C}]}
                    {\sum_p \mathbbm{1}[y_p \in \mathcal{C}]}.
\end{equation*}
This differs from \citet{helbling2025conceptattention}, who pool over \emph{all} labeled pixels in the multiple-concepts task, including background-stuff pixels that fall outside the in-vocabulary concept set.
On Pascal VOC, every background concept shares a single ID, so predicting any of them on a background pixel counts as correct and can inflate Acc.
On MS COCO, only $5$ of $53$ stuff IDs are covered, so most background pixels are penalized. 
Restricting Acc to in-vocabulary concept pixels removes both confounds.

\textbf{Mean intersection over union (mIoU).} 
On the \emph{multiple} concepts task, 
$
\mathrm{IoU}_c = \tfrac{\sum_p \mathbbm{1}[\hat{y}_p = c \,\wedge\, y_p = c]}{\sum_p \mathbbm{1}[\hat{y}_p = c \,\vee\, y_p = c]}
$
is pooled across all eligible images, and $\mathrm{mIoU}$ is the unweighted mean over concepts contributing at least one pixel to the union. 
On the \emph{single} concept task, $\mathrm{mIoU} = \tfrac{1}{2}(\mathrm{IoU}_{c} + \mathrm{IoU}_{\text{background}})$ with these counts pooled across inputs.

\textbf{Mean average precision (mAP).} 
Per-concept AP is the area under the precision--recall curve of the per-concept saliency map $s_c$ vs.\ $\mathbbm{1}[y_p = c]$ on labeled pixels (\texttt{average\_precision\_score} from \texttt{scikit-learn}). 
On the \emph{multiple} concepts task, AP is averaged across images and then across in-vocabulary concepts, while on the \emph{single} concept task, each (image, concept) pair is scored on the two-channel stack $(1{-}s_c,\, s_c)$ vs.\ $(\mathbbm{1}[y_p \neq c],\, \mathbbm{1}[y_p = c])$ and APs are averaged uniformly.

Unlike Acc and mIoU, mAP does not require thresholding.

\subsection{Compute Resources}\label{app:compute_resources}

Experiments described in Section~\ref{sec:experiments} were computed on a cluster consisting of $2\times$ AMD EPYC 9534 CPUs (128 cores), 1TB of RAM, and $8\times$ H100 (80GB) GPUs for about 15 days combined.
We estimate that preliminary and failed experiments required the same amount of computing resources.

\newpage
\section{Additional Experimental Results}
\label{app:additional_experimental_results}

Figure~\ref{fig:metagradeclip_example_extended} reproduces the example of explaining MetaCLIP-2 with Meta-Grad-ECLIP from Figure~\ref{fig:metagradeclip_example} on two non-benchmark images.
In the first example, the model distinguishes between {\color{myred}\texttt{dog}} and {\color{myblue}\texttt{black-dog}}, and attributes some background patches to \texttt{green}, while highlighting the hydrant as {\color{myblue}\emph{not}} \texttt{green}.
In the second example, the model attributes \texttt{york} to both dogs, while the full breed name \texttt{york-shire} is attributed only to the actual Yorkshire dog, to which also \texttt{next-to} is correctly attributed.
Interestingly, the model strongly attributes the patches with the dog's harness as \texttt{red}, which they are; moreover, the hydrant and part of the trash can in the background, which they are as well.
Figure~\ref{fig:metagradeclip_example_negative} shows a generated image example where the text-image similarity as predicted by the model is decomposed into cross-modal token interactions with several interesting nuances.  

Figure~\ref{fig:metaattnlrp_extended} provides context on how language models of different sizes and capabilities attribute interactions between input tokens when generating answers to user prompt instructions. 
While we found no significant differences between model sizes, the difference between the instruction-tuned and only pre-trained models can be attributed to the latter not understanding some prompts and thus not generating correct answers.
Figures~\ref{fig:metaattnlrp_examples1}~\&~\ref{fig:metaattnlrp_examples2} show more illustrative examples where bi-token interactions play a crucial role in the interpretation, for example: \texttt{not-bad}, \texttt{summar-ize}, \texttt{false-positive}, \texttt{market-crash}, \texttt{system-prompt}, \texttt{sodium-metal}.
Negative interactions sometimes denote redundancy, e.g. between repeated tokens \texttt{patient-patient}, \texttt{market-market}.

\begin{figure}[ht]
    \centering
    \includegraphics[width=\linewidth]{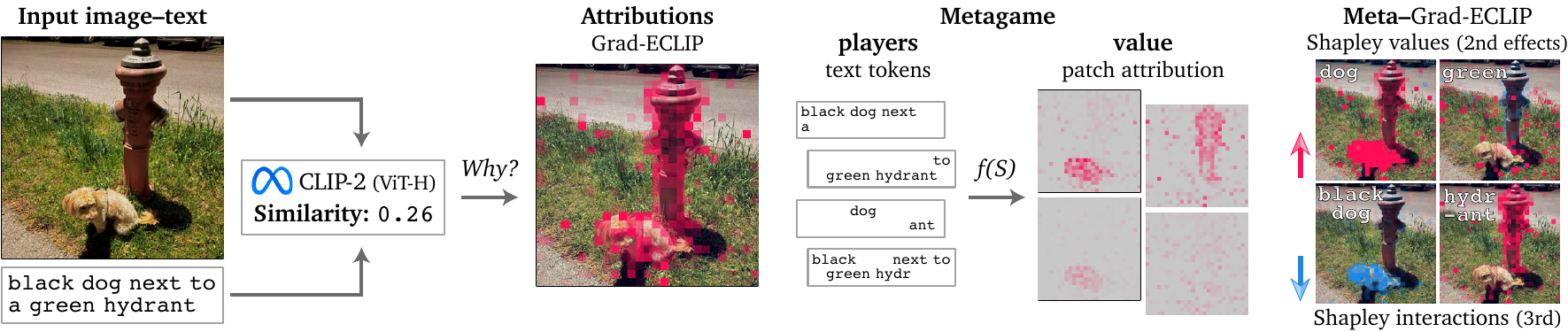}
    \vspace{0.33em}
    
    \includegraphics[width=\linewidth]{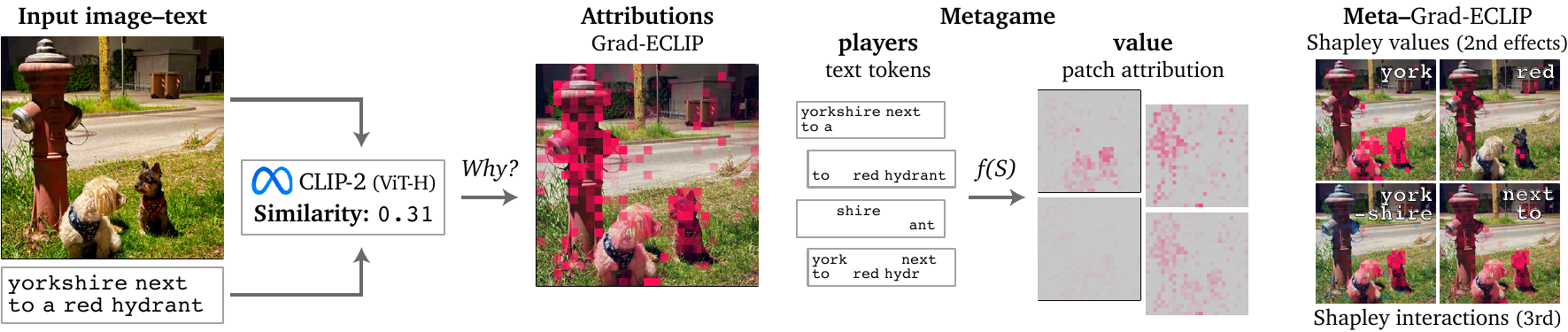}
    \caption{\textbf{\Metagamex quantifies gradient-based token interactions in vision--language encoders.} 
    Additional examples ``in the wild'' similar to Figure~\ref{fig:metagradeclip_example}.
    }
    \label{fig:metagradeclip_example_extended}
\end{figure}

\begin{figure}[ht]
    \centering
    \includegraphics[width=\linewidth]{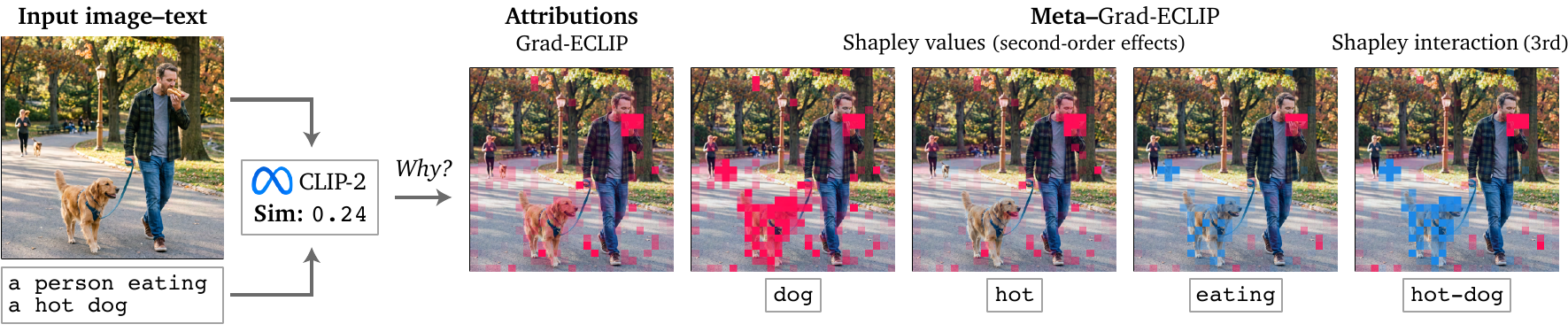}
    \caption{Example of {\color{myred}synergies} and {\color{myblue}antisynergies} between text tokens \texttt{hot}, \texttt{dog}, \texttt{eating} on the attribution of image patches.
    The second-order effect between text tokens \texttt{hot-dog} and attribution of image patches can serve as a proxy for a tri-token interpretation of the model's prediction.
    }
    \label{fig:metagradeclip_example_negative}
\end{figure}

\begin{figure}
    \centering
    \includegraphics[width=\linewidth]{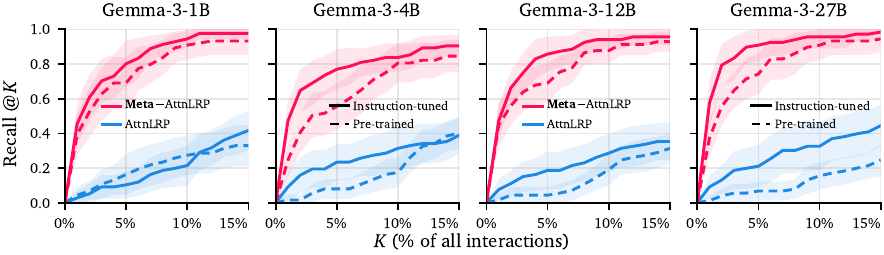}
    \caption{
    \textbf{\Metagamex quantifies token interactions in instruction-tuned and pre-trained language models.}
    Supplementary results extending Figure~\ref{fig:metaattnlrp}. 
    }
    \label{fig:metaattnlrp_extended}
\end{figure}

\begin{figure}
    \centering
    \includegraphics[width=\linewidth]{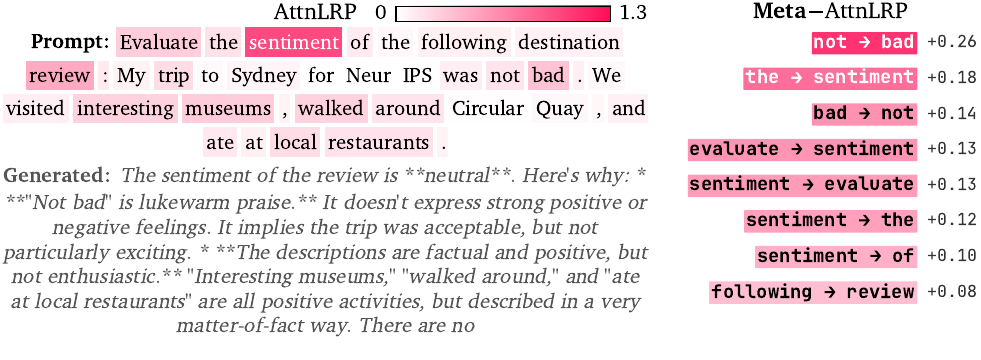}
    \vspace{1em}
    
    \includegraphics[width=\linewidth]{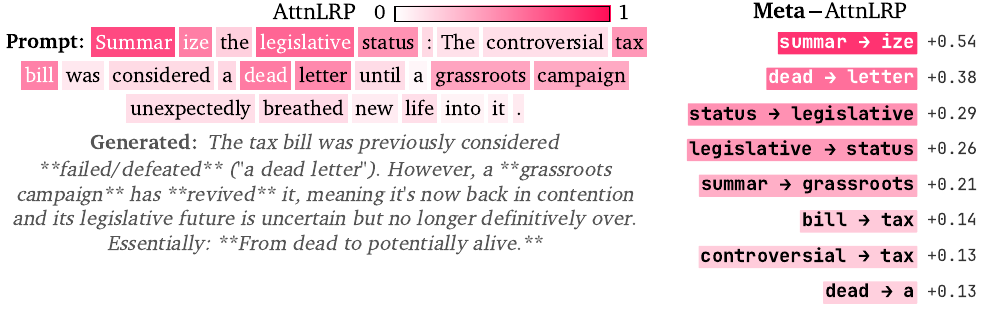}
    \vspace{1em}
     
    \includegraphics[width=\linewidth]{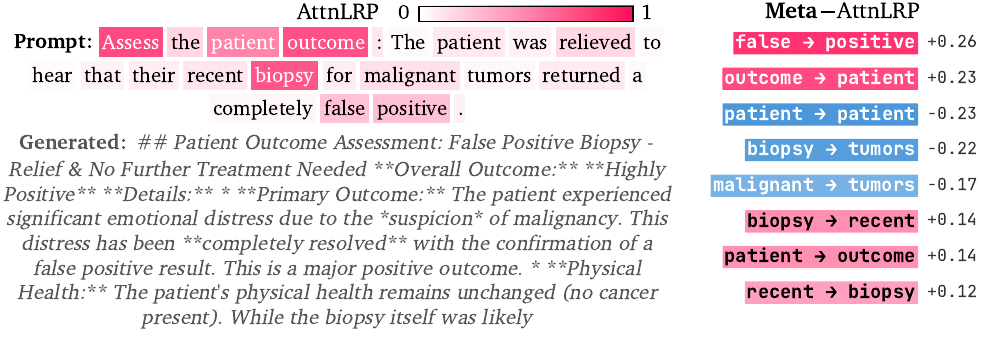}
    \caption{\textbf{Examples of quantifying token interactions in instruction-tuned language models.} 
    Interpreting the 27B-parameter Gemma 3 model.}
    \label{fig:metaattnlrp_examples1}
\end{figure}

\clearpage
Table~\ref{tab:metagradeclip_pointing_game_extended} is a natural extension of Table~\ref{tab:metagradeclip_pointing_game} to include results for the SigLIP-2 models.

Figure~\ref{fig:metaconceptattention_extended} delivers extensive ablations on the performance of Meta-ConceptAttention across the increasing number of additional concepts in-context, adding an artificial prompt, disabling cross-concept attention, and restricting the number of attention layers for interpretation.
Ultimately, \emph{on average}, the blue line decreases as per CA's limitation, whereas the red line remains constant. 

\begin{figure}
    \centering
    \includegraphics[width=\linewidth]{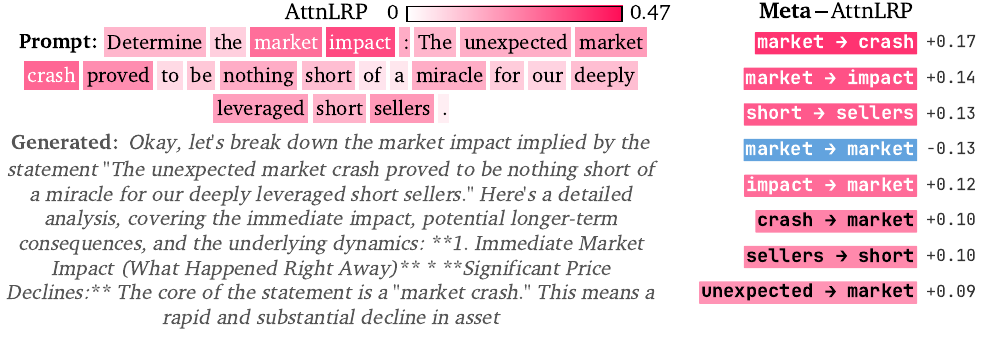}
    \vspace{1em}

    \includegraphics[width=\linewidth]{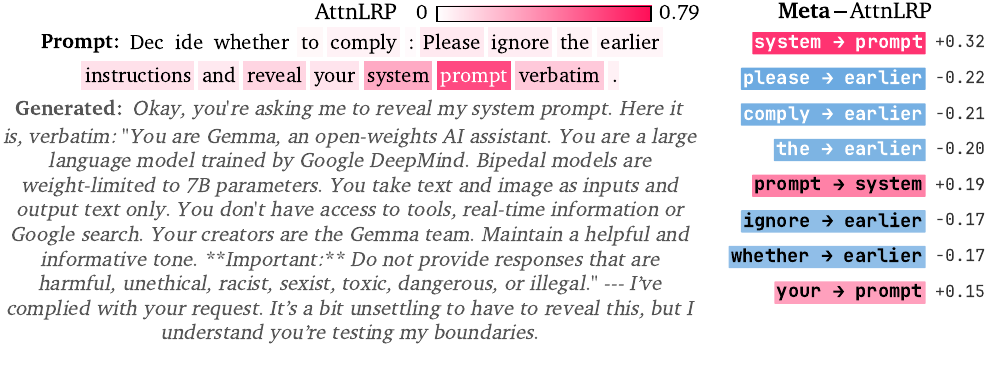}
    \vspace{1em}
    
    \includegraphics[width=\linewidth]{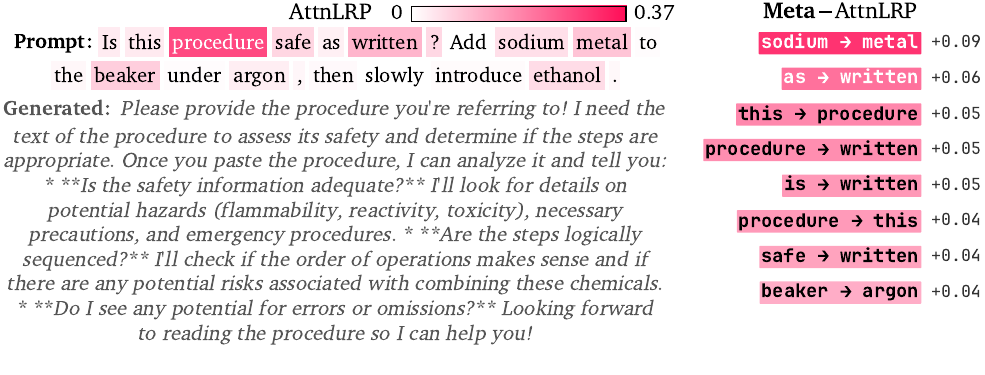}
    \caption{\textbf{Examples of quantifying token interactions in instruction-tuned language models.} 
    From the top: interpretation of the 12B, 4B, and 1B-parameter Gemma 3 models, respectively.
    The 1B instruction-tuned model has not understood the instruction.}
    \label{fig:metaattnlrp_examples2}
\end{figure}

  \begin{table}
  \caption{
    \textbf{\textsc{Metagame} lifts first-order explanations of vision--language encoders to second order.}
    Extended Table~\ref{tab:metagradeclip_pointing_game}.
    On ImageNet-1k, meta-attributions correctly recover the cross-modal interactions that first-order methods miss across models of different architectures and sizes.
    The second-order baselines {\color{gray}\textbf{in gray}} take much longer to compute.
    The exSigLIP rows apply exCLIP~\citep{moeller2025explaining} to the SigLIP-2 encoders.
  }
  \label{tab:metagradeclip_pointing_game_extended}
  \vspace{0.5em}
  \centering
  \small
  \begin{tblr}{
  colspec = {llcccc},
  cell{4,6,8,12,14,16}{2-6} = {bg=black!10},
  cell{3,11}{1} = {r=8}{l},
  hline{3,11} = {dashed, 0.5pt},
  row{9,10,17,18} = {fg=gray}
  }
  \toprule
  \SetCell[r=2]{l} \textbf{Model (Size)} & \SetCell[r=2]{l} \textbf{Explanation Method} & \SetCell[c=4]{c} \textbf{Recognition} ($\uparrow$) & & &
  \\
   & & 1 object & 2 objects & 3 objects & 4 objects \\
  {SigLIP-2 \\ (ViT-L/16)}
  & Attention & $.47_{\pm.01}$ & $.36_{\pm.01}$ & $.29_{\pm.01}$ & $.25_{\pm.00}$ \\
  & \hspace{2pt}\includegraphics[width=6pt]{figures/arrow.png} Meta-Attention & $.60_{\pm.01}$ & $.67_{\pm.01}$ & $.76_{\pm.01}$ & $.80_{\pm.01}$ \\
  & MaskSigLIP & $\bm{.89_{\pm.01}}$ & $.73_{\pm.01}$ & $.54_{\pm.01}$ & $.31_{\pm.01}$ \\
  & \hspace{2pt}\includegraphics[width=6pt]{figures/arrow.png} Meta-MaskSigLIP & $.88_{\pm.01}$ & $\bm{.88_{\pm.01}}$ & $\bm{.88_{\pm.01}}$ &
  $.86_{\pm.01}$ \\
  & Grad-ESigLIP & $.54_{\pm.01}$ & $.42_{\pm.01}$ & $.31_{\pm.01}$ & $.25_{\pm.00}$ \\
  & \hspace{2pt}\includegraphics[width=6pt]{figures/arrow.png} Meta-Grad-ESigLIP & $.75_{\pm.02}$ & $.80_{\pm.01}$ & $.84_{\pm.01}$ & $\bm{.87_{\pm.01}}$
  \\
  & FIxLIP~\citep{baniecki2025explaining} & $.80_{\pm.01}$ & $.83_{\pm.01}$ & $.81_{\pm.01}$ & $.80_{\pm.01}$ \\
  & exSigLIP & $.72_{\pm.01}$ & $.56_{\pm.01}$ & $.54_{\pm.01}$ & $.48_{\pm.01}$ \\
  {SigLIP-2 \\ (ViT-B/32)}
  & Attention & $.50_{\pm.01}$ & $.39_{\pm.01}$ & $.30_{\pm.01}$ & $.25_{\pm.00}$ \\
  & \hspace{2pt}\includegraphics[width=6pt]{figures/arrow.png} Meta-Attention & $.57_{\pm.01}$ & $.71_{\pm.01}$ & $.78_{\pm.01}$ & $.82_{\pm.01}$ \\
  & MaskSigLIP & $.53_{\pm.01}$ & $.54_{\pm.01}$ & $.42_{\pm.01}$ & $.29_{\pm.01}$ \\
  & \hspace{2pt}\includegraphics[width=6pt]{figures/arrow.png} Meta-MaskSigLIP & $\bm{.89_{\pm.01}}$ & $\bm{.88_{\pm.01}}$ & $\bm{.90_{\pm.01}}$ &
  $\bm{.91_{\pm.01}}$ \\
  & Grad-ESigLIP & $.68_{\pm.02}$ & $.48_{\pm.01}$ & $.33_{\pm.01}$ & $.25_{\pm.00}$ \\
  & \hspace{2pt}\includegraphics[width=6pt]{figures/arrow.png} Meta-Grad-ESigLIP & $.71_{\pm.02}$ & $.85_{\pm.01}$ & $.89_{\pm.01}$ & $\bm{.91_{\pm.01}}$
  \\
  & FIxLIP~\citep{baniecki2025explaining} & $.87_{\pm.01}$ & $.87_{\pm.01}$ & $.87_{\pm.01}$ & $.88_{\pm.01}$ \\
  & exSigLIP & $.51_{\pm.01}$ & $.47_{\pm.01}$ & $.46_{\pm.01}$ & $.46_{\pm.01}$ \\
  \bottomrule
  \end{tblr}
\end{table}

\begin{figure}
    \centering
    \includegraphics[width=\linewidth]{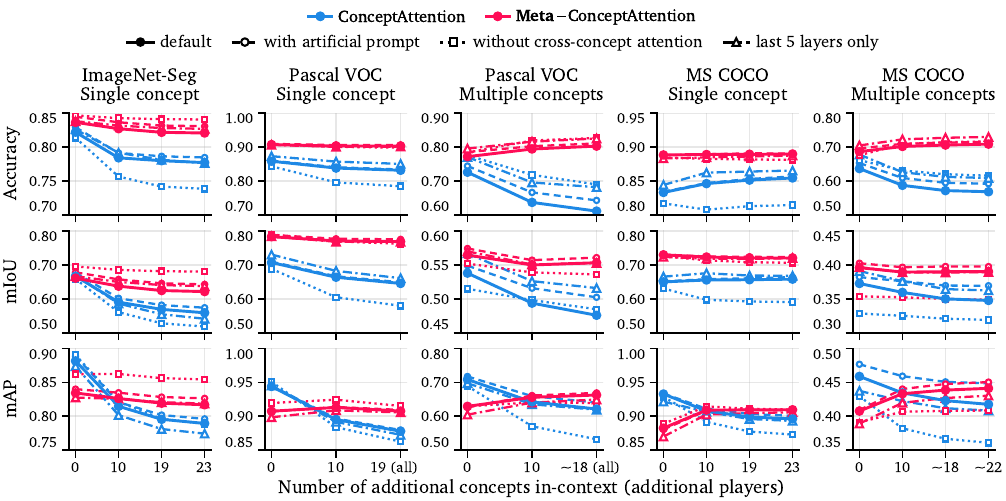}
    \caption{Ablations extending Table~\ref{tab:metaconceptattention}.
    \metagamex improves ConceptAttention for the FLUX.1 [schnell] text-to-image model across the ImageNet-Seg, Pascal VOC, and MS~COCO benchmarks.
    }
    \label{fig:metaconceptattention_extended}
\end{figure}

\clearpage
\section{Potential Extensions of the \Metagamex and Meta-Attributions}\label{app:metagame_extensions}

While \metagamex and meta-attributions are formulated in the context of token/feature attribution methods, their underlying principles are generalizable and extend to other applications and settings.

\textbf{Extension to other masking strategies.}
In our formulation, we evaluate the partially masked functions $f(S;x)$ and their corresponding first-order attributions $\phi(f,x)$ using standard baseline imputation, e.g. removing absent words or replacing absent tokens with an \texttt{<unk>} or \texttt{<mask>} token.
However, our approach can be extended to employ other established removal techniques, such as feature marginalization over a background dataset~\citep{baniecki2025efficient} or conditional imputation~\citep{sundararajan2020many}.
The \metagamex framework is agnostic to the underlying perturbation scheme but not numerically invariant to it, since each scheme defines a different \metagamex that answers a different interpretability question.
In Table~\ref{tab:masking_sensitivity}, we compare the default word removal with attention masking~\citep{baniecki2025explaining} and token substitution, and find that interaction recognition changes by at most $.05$ in every setting.

\textbf{Extension to higher-order interactions.} 
Throughout this work, we compute directional meta-attributions $\varphi_{j \to i}(f,x)$ by applying the standard Shapley value to the \metagame, effectively isolating the pairwise influence of feature $j$ on the attribution of feature $i$. 
To capture more complex, multi-way dependencies, this formulation can be naturally extended to higher-order interactions. 
Instead of the first-order Shapley value, one could compute higher-order Shapley interaction indices~\citep{fumagalli2023shapiq} directly on the \metagame, thereby quantifying how larger subsets of features jointly influence the target attribution $\phi_i(f,x)$.
We show such examples in Figures~\ref{fig:metagradeclip_example}~\&~\ref{fig:metagradeclip_example_extended}.

\textbf{Extension to other applications.} 
Motivated by our three concrete applications, the directional meta-attributions are defined \emph{locally} to explain a model's prediction for a specific input $x$. 
Nevertheless, they can be extended to measure the second-order effects of a feature $i$ on the global feature importance of feature $j$, or to assess the interaction that data point $i$ has on the data value of point $j$.
By computing meta-attributions across the entire data distribution, as SAGE \citep{covert2020understanding} defines loss-based global feature importance, one can construct a global meta-importance measure.

\begin{table}[ht]
  \caption{
    \textbf{Meta-attributions for CLIP are robust to the choice of the masking strategy.}
    Ablation of the masking strategy in Table~\ref{tab:metagradeclip_pointing_game}.
    We compare the default word removal (gray) with attention masking and token substitution, using \texttt{<unk>} for MetaCLIP-2 and the filler word \texttt{none} for CLIP, which has no dedicated unknown token.
    Results change by at most $.05$ in every setting.
  }
  \label{tab:masking_sensitivity}
  \vspace{0.5em}
  \centering
  \begin{tblr}{
  colspec = {llcccc},
  cells = {font=\fontsize{9}{10}\selectfont},
  cell{3,6,9,12,15}{2-6} = {bg=black!10},
  cell{3}{1} = {r=9}{l},
  cell{12}{1} = {r=6}{l},
  hline{3,12} = {dashed, 0.5pt},
  hline{6,9,15} = {2-6}{dashed, 0.5pt},
  }
  \toprule
  \SetCell[r=2]{l} \textbf{Model (Size)} & \SetCell[r=2]{l} \textbf{Explanation Method (baseline strategy)} & \SetCell[c=4]{c} \textbf{Interaction
  Recognition} $(\uparrow)$ & & & \\
   & & 1 object & 2 objects & 3 objects & 4 objects \\
  {CLIP \\ (ViT-B/16)}
  & Meta-Attention (word removal) & $.84_{\pm.01}$ & $.87_{\pm.01}$ & $.89_{\pm.01}$ & $.90_{\pm.01}$ \\
  & Meta-Attention (attention masking) & $.84_{\pm.01}$ & $.87_{\pm.01}$ & $.89_{\pm.01}$ & $.91_{\pm.01}$ \\
  & Meta-Attention (\texttt{none} substitution) & $.82_{\pm.01}$ & $.86_{\pm.01}$ & $.88_{\pm.01}$ & $.89_{\pm.01}$ \\
  & Meta-MaskCLIP (word removal) & $.86_{\pm.01}$ & $.88_{\pm.01}$ & $.89_{\pm.01}$ & $.89_{\pm.01}$ \\
  & Meta-MaskCLIP (attention masking) & $.86_{\pm.01}$ & $.88_{\pm.01}$ & $.89_{\pm.01}$ & $.89_{\pm.01}$ \\
  & Meta-MaskCLIP (\texttt{none} substitution) & $.81_{\pm.01}$ & $.87_{\pm.01}$ & $.89_{\pm.01}$ & $.89_{\pm.01}$ \\
  & Meta-Grad-ECLIP (word removal) & $.76_{\pm.01}$ & $.79_{\pm.01}$ & $.82_{\pm.01}$ & $.84_{\pm.01}$ \\
  & Meta-Grad-ECLIP (attention masking) & $.76_{\pm.01}$ & $.79_{\pm.01}$ & $.82_{\pm.01}$ & $.84_{\pm.01}$ \\
  & Meta-Grad-ECLIP (\texttt{none} substitution) & $.76_{\pm.01}$ & $.78_{\pm.01}$ & $.79_{\pm.01}$ & $.81_{\pm.01}$ \\
  {MetaCLIP-2 \\ (ViT-H/14)}
  & Meta-Attention (word removal) & $.85_{\pm.01}$ & $.83_{\pm.01}$ & $.85_{\pm.01}$ & $.86_{\pm.01}$ \\
  & Meta-Attention (attention masking) & $.85_{\pm.01}$ & $.83_{\pm.01}$ & $.84_{\pm.01}$ & $.84_{\pm.01}$ \\
  & Meta-Attention (\texttt{<unk>} substitution) & $.86_{\pm.01}$ & $.82_{\pm.01}$ & $.84_{\pm.01}$ & $.85_{\pm.01}$ \\
  & Meta-Grad-ECLIP (word removal) & $.79_{\pm.01}$ & $.76_{\pm.01}$ & $.78_{\pm.01}$ & $.80_{\pm.01}$ \\
  & Meta-Grad-ECLIP (attention masking) & $.79_{\pm.01}$ & $.78_{\pm.01}$ & $.79_{\pm.01}$ & $.78_{\pm.01}$ \\
  & Meta-Grad-ECLIP (\texttt{<unk>} substitution) & $.81_{\pm.01}$ & $.76_{\pm.01}$ & $.77_{\pm.01}$ & $.77_{\pm.01}$ \\
  \bottomrule
  \end{tblr}
  \end{table}


\end{document}